\documentclass{article}
\usepackage{amsmath,amsfonts,bm}

\def\eqref#1{equation~\ref{#1}}
\def\1{\bm{1}}

\def\vmu{{\bm{\mu}}}
\def\vtheta{{\bm{\theta}}}
\def\vsigma{{\bm{\sigma}}}
\def\vpsi{{\bm{\psi}}}
\def\va{{\bm{a}}}

\def\vg{{\bm{g}}}
\def\vh{{\bm{h}}}

\def\vm{{\bm{m}}}

\def\vp{{\bm{p}}}
\def\vq{{\bm{q}}}
\def\vr{{\bm{r}}}
\def\vs{{\bm{s}}}

\def\vu{{\bm{u}}}
\def\vv{{\bm{v}}}

\def\vx{{\bm{x}}}

\def\vz{{\bm{z}}}

\def\mA{{\bm{A}}}
\def\mB{{\bm{B}}}

\def\mD{{\bm{D}}}

\def\mF{{\bm{F}}}
\def\mG{{\bm{G}}}
\def\mH{{\bm{H}}}
\def\mI{{\bm{I}}}

\def\mM{{\bm{M}}}

\def\mS{{\bm{S}}}

\def\mU{{\bm{U}}}
\def\mV{{\bm{V}}}
\def\mW{{\bm{W}}}
\def\mX{{\bm{X}}}

\DeclareMathAlphabet{\mathsfit}{\encodingdefault}{\sfdefault}{m}{sl}
\SetMathAlphabet{\mathsfit}{bold}{\encodingdefault}{\sfdefault}{bx}{n}

\newcommand{\E}{\mathbb{E}}

\usepackage{microtype}
\usepackage{graphicx}
\usepackage{subfigure}
\usepackage{booktabs} % for professional tables

\usepackage{hyperref}

\usepackage[accepted]{icml2024}

\usepackage{amsmath}
\usepackage{amssymb}
\usepackage{mathtools}
\usepackage{amsthm}

\usepackage{xfrac}
\usepackage{bbm}
\newcommand{\shorteqnote}[1]{ & & & & & \text{\small\llap{#1}}}

\usepackage[capitalize,noabbrev]{cleveref}

\theoremstyle{plain}
\newtheorem{theorem}{Theorem}[section]

\theoremstyle{definition}

\theoremstyle{remark}

\usepackage[textsize=tiny]{todonotes}

\allowdisplaybreaks

\icmltitlerunning{Revisiting Scalable Hessian Diagonal Approximations}

\begin{document}

\twocolumn[
\icmltitle{Revisiting Scalable Hessian Diagonal Approximations\\ for Applications in Reinforcement Learning}

\icmlsetsymbol{equal}{*}

\begin{icmlauthorlist}
\icmlauthor{Mohamed Elsayed}{uofa,amii}
\icmlauthor{Homayoon Farrahi}{uofa,amii}
\icmlauthor{Felix Dangel}{vector}
\icmlauthor{A. Rupam Mahmood}{uofa,amii,cifar}
\end{icmlauthorlist}

\icmlaffiliation{uofa}{Department of Computing Science, University of Alberta, Edmonton, Canada}
\icmlaffiliation{vector}{Vector Institute, Toronto, Canada}
\icmlaffiliation{cifar}{CIFAR AI Chair}
\icmlaffiliation{amii}{Alberta Machine Intelligence Institute}

\icmlcorrespondingauthor{Mohamed Elsayed}{mohamedelsayed@ualberta.ca}

\icmlkeywords{Machine Learning, ICML}

\vskip 0.3in
]

\printAffiliationsAndNotice{}

\begin{abstract}
Second-order information is valuable for many applications but challenging to compute.
Several works focus on computing or approximating Hessian diagonals, but even this simplification introduces significant additional costs compared to computing a gradient. 
In the absence of efficient exact computation schemes for Hessian diagonals, we revisit an early approximation scheme proposed by 
\citet[\emph{BL89}]{becker1988improving}, which has a cost similar to gradients and appears to have been overlooked by the community.
We introduce \emph{HesScale}, an improvement over BL89, which adds negligible extra computation.
On small networks, we find that this improvement is of higher quality than all alternatives, even those with theoretical guarantees, such as unbiasedness, while being much cheaper to compute. 
We use this insight in reinforcement learning problems where small networks are used and demonstrate HesScale in second-order optimization and scaling the step-size parameter.
In our experiments, HesScale optimizes faster than existing methods and improves stability through step-size scaling.
These findings are promising for scaling second-order methods in larger models in the future.\footnote{Code is available at:\\ \url{https://github.com/mohmdelsayed/HesScale}}
\end{abstract}

%%%%%%%%%%%%%%%%%%%%%%%%%%%%%%%%%%%

\section{Introduction}

Second-order information---the entries of the Hessian matrix---is paramount in a wide spectrum of applications, including preconditioning in optimization (\citealt{martens2015optimizing, yao2021adahessian, shen2024ivon}), and estimating the importance of weights or neurons \cite{elsayed2024addressing} in pruning (\citealt{lecun1989optimal,hassibi1992second, singh2020woodfisher}). However, computing the Hessian entries is generally expensive, preventing its usage in systems that require small memory and computation. This is reflected in the minimal adoption of second-order methods compared to their first-order counterparts.

Many second-order methods rely on some approximation of the Hessian entries to make computation less prohibitive. For example, a type of truncated-Newton method called Hessian-free methods \cite{martens2010deep} exploits the Hessian-vector product \cite{pearlmutter1994fast, bekas2007estimator}. However, such methods might require many iterations per update or additional techniques to achieve stability when used in optimization, adding computational overhead \cite{martens2011learning}. Some variations only approximate the diagonals of the Hessian matrix using stochastic estimation with matrix-free computations (e.g., \citealt{chapelle2011improved,martens2012estimating,yao2021adahessian}), but they may produce low-quality approximations of the Hessian entries \citep{jahani2021doubly}. Other methods impose probabilistic modeling assumptions and estimate a block diagonal Fisher information matrix \cite{martens2015optimizing,botev2017practical}, but they are more expensive to compute.

A long-overlooked approach is deterministic diagonal approximations to the Hessian. Specifically, the method proposed by \citet{becker1988improving}, which we call \emph{BL89}, can be implemented as efficiently as first-order methods. We call this method, along with other cheap methods (e.g., \citealp{yao2021adahessian}) with the same computational and memory complexities as the gradient, \textit{scalable second-order methods}, distinguishing them from methods with superlinear computational or memory complexity (e.g., \citealp{mizutani2008second,botev2017practical, dangel2020modular}). Despite the promise of scalable second-order optimization, the approximation quality of BL89 is shown to be poor \cite{hassibi1992second,martens2012estimating}. A scalable second-order method with a high-quality approximation is still needed.

Scalable second-order methods can benefit reinforcement learning (RL) in at least three ways: sample efficiency, stability, and robustness. Second-order optimization can help with sample efficiency due to faster convergence \cite{wu2017scalable}. Moreover, step-size scaling schemes (e.g., \citealp{martens2015optimizing}) use Hessian approximations to improve the stability of reinforcement learning methods that can suffer from instability due to large updates \cite{dohare2023overcoming}. Another benefit of step-size scaling is achieving robustness against step-size sensitivity, allowing for tuning-free optimization \cite{mahmood2012tuning}, which is crucial for learning with real-life robots \cite{mahmood2018benchmarking}.

In this paper, we introduce a refinement of BL89, which we call \textit{HesScale}. Similarly to BL89, our method is scalable, with small computational and memory overhead, but maintains higher approximation accuracy than BL89 and many other methods for Hessian diagonal approximation. We provide two applications of HesScale: second-order optimization and step-size scaling. In supervised learning and RL tasks, we demonstrate that second-order optimization with HesScale achieves better sample efficiency compared to existing scalable second-order methods.
In RL tasks, we show that step-size scaling with HesScale improves both robustness and stability of the base learning method.

%%%%%%%%%%%%%%%%%%%%%%%%%%%%%%%%%%%

\section{Background}
Here, we describe the Hessian matrix for neural networks and some existing methods for estimating it.
Generally, Hessian matrices can be computed for any scalar-valued function that is twice differentiable.
If $f: \mathbb{R}^n \rightarrow \mathbb{R}$ is such a function, then for its argument $\vpsi \in \mathbb{R}^n$, the Hessian matrix $\mH\in \mathbb{R}^{n\times n}$ of $f$ with respect to $\vpsi$ is given by $H_{i,j} = \sfrac{\partial^2 f(\vpsi)}{\partial \psi_i \partial \psi_j}$.
Here, element $i$ of vector $\vv$ is denoted by $v_i$, and element $(i, j)$ of matrix $\mM$ is denoted by $M_{i,j}$.
For optimization in deep learning, the function $f$ is typically the objective function, and 
the vector $\vpsi$ is commonly the weight vector of a neural network.
Computing and storing an $n \times n$ matrix, where $n$ is the number of weights in a neural network, is prohibitively expensive. Therefore, many methods exist for approximating the Hessian matrix or parts of it with less memory footprint, computational requirement, or both. A common technique is to utilize the structure of the function to reduce the computations needed. For example, some approximate a layer-wise block diagonal Hessian. The computation further simplifies when activation functions are assumed to be piece-wise linear. This assumption results in the \emph{Generalized Gauss-Newton} (GGN, \citealp{schraudolph2002fast}) approximation. However, computing and storing the GGN matrix or its block diagonals is still too demanding.

Many approximation methods were developed to reduce the storage and computation requirements of the block-diagonal GGN matrix. For example, under probabilistic modeling assumptions, the \emph{Kronecker-factored Approximate Curvature} (KFAC, \citealp{martens2015optimizing}) method writes each GGN's diagonal block $\mG$ as a Kronecker product of two matrices of smaller sizes as: $\mG\approx\mA \otimes \mB$, where $\mA = \mathbb{E}[\vh\vh^{\top}]$, $\mB = \mathbb{E}[\vg\vg^{\top}]$, $\vh$ is the activation vector, and $\vg$ is the gradient of the loss with respect to the pre-activation vector. The $\mA$ and $\mB$ matrices can be estimated by Monte Carlo estimation. KFAC is also more efficient when used in optimization than other methods approximating GGN block diagonals since it requires inverting only the small matrices using the Kronecker product property $(\mA \otimes \mB)^{-1}=\mA^{-1} \otimes \mB^{-1}$.
Despite KFAC's gains in efficiency, it is still costly since storing the Kronecker matrices can become prohibitively expensive for large-scale problems. Additionally, its Kronecker structure introduces approximation errors.
Alternative approaches that achieve high accuracy exploit the GGN's outer-product structure \cite{dangel2022vivit,yang2022sketch,dangel2019backpack}, but they suffer from a superlinear scaling in the network's output dimension in both memory and compute, which limits their ability to scale.

Restricting calculations to Hessian diagonals provides some curvature information with relatively little computation. However, it has been shown that the exact computation for diagonals of the Hessian typically has quadratic complexity with the unlikely existence of algorithms that can compute the exact diagonals with less than quadratic complexity \cite{martens2012estimating}. 
Some stochastic methods provide a way to compute unbiased estimates of the exact Hessian diagonals. For example, the AdaHessian \cite{yao2021adahessian} algorithm uses Hutchinson’s estimator $\operatorname{diag}(\mH) = \mathbb{E}[\vz \circ (\mH\vz)]$, where $\vz$ is a multivariate random variable with a Rademacher distribution and the expectation can be estimated using Monte Carlo estimation. 
Similarly, the GGN-MC method \cite{dangel2019backpack} uses the relationship between the Fisher information matrix and the Hessian matrix under probabilistic modeling assumptions to have an MC approximation of the diagonal of the GGN matrix. Although these stochastic approximation methods are scalable, that is, with linear computational and memory complexity in the number of parameters and network outputs, they suffer from low approximation quality (see Fig.\ \ref{fig:approximation_quality}), improving which requires many sampling and factors of additional computations.

%%%%%%%%%%%%%%%%%%%%%%%%%%%%%%%%%%%

\section{The Proposed HesScale Method}
\label{method-section}
We present our method for approximating the diagonal of the Hessian at each layer in feed-forward networks, where a backpropagation rule is used to utilize the Hessian of previous layers.
Here, we derive the backpropagation rule for fully connected networks. A similar derivation for fully connected networks with the mean squared error was presented for BL89 \cite{becker1988improving}. However, ours is a refinement of BL89 in that we use the exact diagonals of the Hessian matrix at the last layer. We show that the computational complexity can still be linear in the network's output dimension for some common loss functions. We defer the derivation of Hessian diagonals for the convolutional neural networks to Appendix \ref{appendix:cnn-derivation}.

We use the non-convex stochastic optimization setting where there is an objective we need to minimize. The stochastic objective given a sample $S$ is denoted by $\mathcal{L}(S, \mathcal{W})$, where $S$ is a random variable that can be the input-output pair in supervised learning or a transition tuple in reinforcement learning, and $\mathcal{W}$ is a set of learnable parameters. The \emph{learner} maintains a single or multiple neural network and is required to optimize the objective by changing these network parameters. Specifically, the learner is required to minimize the objective $ \mathbb{E}_S(\mathcal{L}(S, \mathcal{W}))$.

\vspace{-0.2cm}
\begin{figure}[ht]
  \centering
  \includegraphics[width=0.9\columnwidth]{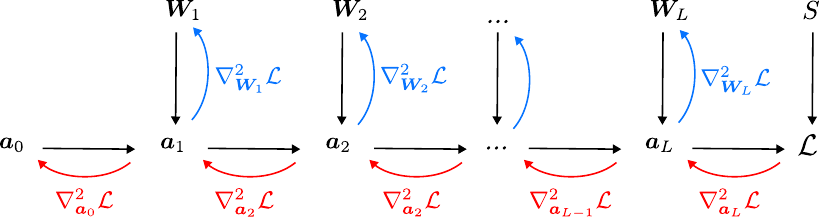}
   \vspace{-0.3cm}
  \caption{Backpropagating the exact Hessian information in a neural network. {\color{red}Red} arrows represent the direction of influence while backpropagating the Hessian of the loss w.r.t. pre-activations which are then used to compute the Hessian of the loss w.r.t. the weights at each layer, denoted by the {\color{blue}blue}  arrows.
  Black arrows denote the direction of influence during the forward pass.
}
  \label{fig:illustration}
\end{figure}

Consider a neural network with $L$ layers parametrized by the set of weights $\mathcal{W}=\{\mW_1,...,\mW_L\}$, where $\mW_l$ is the weight matrix at the $l$-th layer, and its element at the $i$th row and the $j$th column is denoted by $W_{l,i,j}$. 
During learning, the parameters of the neural network are changed to reduce the loss. 
During a forward pass, we get the activation $\vh_{l}$ at layer $l$ by applying the activation function $\boldsymbol\sigma$ to the pre-activation $\va_{l}$: $\vh_{l} = \boldsymbol\sigma(\va_{l})$. We simplify notations by defining $\vh_0 \doteq \vx$.
The activation $\vh_{l}$ is then multiplied by the weight matrix $\mW_{l+1}$ of layer $l+1$ to produce the next pre-activation: ${a}_{l+1,i} = \sum_{j=1}^{|\vh_{l}|} {{W}_{l+1,i,j}}{h}_{l,j}$.\footnote{The bias term can be added by appending an additional column to a weight matrix and a 1 to each layer's input vector.} 
We assume here that the activation function is element-wise activation for all layers except for the final layer $L$, where it becomes the softmax function. 
The backpropagation equations for the described network are given as \cite{rumelhart1986learning}:
\begin{align}
\frac{\partial \mathcal{L}}{\partial W_{l,i,j}} &= \frac{\partial \mathcal{L}}{\partial a_{l,i}} \frac{\partial a_{l,i}}{\partial W_{l,i,j}} = \frac{\partial \mathcal{L}}{\partial a_{l,i}} h_{l-1,j}, \label{equ:gradient-W}
\end{align}
\begin{align}
\frac{\partial \mathcal{L}}{\partial a_{l,i}} &=  \sum_{k=1}^{|\va_{l+1}|} \frac{\partial \mathcal{L}}{\partial a_{l+1,k}} \frac{\partial a_{l+1,k}}{\partial h_{l,i}} \frac{\partial h_{l,i}}{\partial a_{l,i}} \nonumber \\
&= \sigma^{\prime}(a_{l,i})\sum_{k=1}^{|\va_{l+1}|} \frac{\partial \mathcal{L}}{\partial a_{l+1,k}} W_{l+1, k,i}.\label{equ:gradient-a}
\end{align}
In the following, we write the equations for the exact Hessian diagonals with respect to weights $\sfrac{\partial^2 \mathcal{L}}{\partial {W^2_{l,i, j}}}$, which requires the calculation of $\sfrac{\partial^2 \mathcal{L}}{\partial {a^2_{l,i}}}$. Fig.\ \ref{fig:illustration} shows the general computational graph used to backpropagate the second-order information \cite{mizutani2008second,dangel2020modular}, but we only focus next on the diagonal part:
\begin{align}
\frac{\partial^2 \mathcal{L}}{\partial {W^2_{l,i,j}}}
&= \frac{\partial}{\partial W_{l,i,j}} \left(\frac{\partial \mathcal{L}}{\partial a_{l,i}}h_{l-1,j}\right) \nonumber \\
&= \frac{\partial}{\partial a_{l,i}} \left( \frac{\partial \mathcal{L}}{\partial a_{l,i}}\right) \frac{\partial a_{l,i}}{\partial W_{l,i,j}} h_{l-1,j} \nonumber \\
&= \frac{\partial^2 \mathcal{L}}{\partial a^2_{l,i}} h^2_{l-1,j}\nonumber,
\end{align}
\begin{align}
\frac{\partial^2 \mathcal{L}}{\partial {a^2_{l,i}}}
&= \frac{\partial}{\partial a_{l,i}} \left(\sigma^{\prime}(a_{l,i})\sum_{k=1}^{|\va_{l+1}|} \frac{\partial \mathcal{L}}{\partial a_{l+1,k}} W_{l+1,k,i}   \right) \nonumber \\
&= {\sigma^{\prime}}(a_{l,i}) \sum_{k,p=1}^{|\va_{l+1}|} \frac{\partial^2 \mathcal{L}}{\partial a_{l+1,k} \partial a_{l+1,p}} \frac{\partial a_{l+1,p}}{\partial a_{l,i}} W_{l+1,k,i}  \nonumber \\
&\quad + \sigma^{\prime\prime}(a_{l,i}) \sum_{k=1}^{|\va_{l+1}|} \frac{\partial \mathcal{L}}{\partial a_{l+1,k}} W_{l+1,k,i}  \nonumber \\
&= {\sigma^{\prime}}(a_{l,i})^2 \sum_{k,p=1}^{|\va_{l+1}|} \frac{\partial^2 \mathcal{L}}{\partial a_{l+1,k} \partial a_{l+1,p}} W_{l+1,p,i} W_{l+1,k,i} \nonumber \\
&\quad + \sigma^{\prime\prime}(a_{l,i}) \sum_{k=1}^{|\va_{l+1}|} \frac{\partial \mathcal{L}}{\partial a_{l+1,k}} W_{l+1,k,i}. \nonumber
\end{align}
Since the calculation of $\sfrac{\partial^2 \mathcal{L}}{\partial {a^2_{l,i}}}$ depends on the off-diagonal terms, the computation complexity becomes quadratic in the layer's width. 
Following BL89, we approximate the Hessian diagonals by ignoring the off-diagonal terms, leading to a backpropagation rule with linear computational complexity for our estimates $\widehat{\frac{\partial^2 \mathcal{L}}{\partial {W^2_{l,i, j}}}}$ and $\widehat{\frac{\partial^2 \mathcal{L}}{\partial {a^2_{l,i}}}}$:
\begin{align}
\widehat{\frac{\partial^2 \mathcal{L}}{\partial {W^2_{l,i, j}}}} &\doteq \widehat{\frac{\partial^2 \mathcal{L}}{\partial {a^2_{l,i}}}}h^2_{l-1,j}, \label{equ:hessian-approx-W}\\
\widehat{\frac{\partial^2 \mathcal{L}}{\partial {a^2_{l,i}}}} &\doteq {\sigma^{\prime}}(a_{l,i})^2 \sum_{k=1}^{|\va_{l+1}|}\widehat{\frac{\partial^2 \mathcal{L}}{\partial {a^2_{l+1,k}}}} W^2_{l+1,k,i} \nonumber \\
&\quad + \sigma^{\prime\prime}(a_{l,i}) \sum_{k=1}^{|\va_{l+1}|} \frac{\partial \mathcal{L}}{\partial a_{l+1,k}} W_{l+1,k,i} \label{equ:hessian-approx-a}.
\end{align}
For the last layer, we use the exact Hessian diagonals  $\widehat{\frac{\partial^2 \mathcal{L}}{\partial {a^2_{L,i}}}} \doteq \frac{\partial^2 \mathcal{L}}{\partial {a^2_{L,i}}}$ since it can be computed cheaply for some common loss functions. For example, the exact Hessian diagonals for cross-entropy loss with softmax is simply $\vq-\vq\circ\vq$, where $\vq$ is the predicted probability vector and $\circ$ denotes element-wise multiplication.
We show this property with derivations for negative log-likelihood function with Gaussian and softmax distributions in Appendix \ref{appendix:linear-likelihood-hessian} and for other RL-related loss functions in Appendix \ref{appendix:rl-loss-hesscale}.  

We found empirically that this small change makes a large difference in the approximation quality, as shown in Fig.\ \ref{fig:approx_q_norm}.
Hence, unlike BL89, which uses a Hessian diagonal approximation of the last layer by Eq. \ref{equ:hessian-approx-a}, we use the exact values directly to achieve more approximation accuracy. 
We call this method for Hessian diagonal approximation \emph{HesScale} and provide its pseudocode in Algorithm \ref{alg:hesscale}. Finally, we provide in Appendix \ref{appendix:diagonality-conditions} additional analysis on the special cases that make HesScale exact.

\begin{algorithm}
\caption{HesScale}\label{alg:hesscale}
\begin{algorithmic}
\STATE {\bfseries Require:} Neural network $f$ and a layer number $l$
\STATE {\bfseries Require:} $\frac{\partial \mathcal{L}}{\partial \va_{l+1}}$ and $\widehat{\frac{\partial^2 \mathcal{L}}{\partial \va_{l+1}^2}}$, unless $l=L$
\STATE {\bfseries Require:} Loss function $\mathcal{L}$
\IF{$l=L$} 
    \STATE Compute $\frac{\partial \mathcal{L}}{\partial \va_{L}}$ and $\frac{\partial^2 \mathcal{L}}{\partial \va_{L}^2}$
    \STATE Compute $\frac{\partial \mathcal{L}}{\partial \mW_{L}}$ using Eq.\ \ref{equ:gradient-W} 
    \STATE {\color{blue}$\widehat{\frac{\partial^2 \mathcal{L}}{\partial \va_{L}^2}} \leftarrow \frac{\partial^2 \mathcal{L}}{\partial \va_{L}^2}$}
    \STATE Compute $\widehat{\frac{\partial^2 \mathcal{L}}{\partial \mW_{L}^2}}$ using Eq.\ \ref{equ:hessian-approx-W}
\ELSIF{$l\neq L$}
    \STATE Compute $\frac{\partial \mathcal{L}}{\partial \va_{l}}$ and $\frac{\partial \mathcal{L}}{\partial \mW_{l}}$ using Eq. \ref{equ:gradient-a} and Eq. \ref{equ:gradient-W}
    \STATE Compute $\widehat{\frac{\partial^2 \mathcal{L}}{\partial \va_{l}^2}}$ and $\widehat{\frac{\partial^2 \mathcal{L}}{\partial \mW_{l}^2}}$ using Eq.\ \ref{equ:hessian-approx-a} and Eq.\ \ref{equ:hessian-approx-W}
\ENDIF
\STATE {\bfseries Return} $\frac{\partial \mathcal{L}}{\partial \mW_{l}}$, $\widehat{\frac{\partial^2 \mathcal{L}}{\partial \mW_{l}^2}}$, $\frac{\partial \mathcal{L}}{\partial \va_{l}}$, and $\widehat{\frac{\partial^2 \mathcal{L}}{\partial \va_{l}^2}}$
\end{algorithmic}
\end{algorithm}
\vspace{-2mm}
The computation can be reduced further 
by dropping the last term in Eq.\ \ref{equ:hessian-approx-a}, which corresponds to the Gauss-Newton approximation and is justified for piece-wise linear activation functions.
We call this variation \emph{HesScaleGN}.

%%%%%%%%%%%%%%%%%%%%%%%%%%%%%%%%%%%

\section{Approximation Quality}
We evaluate HesScale's approximation quality and compare it with other methods. We start by studying the approximation quality of Hessian diagonals compared to the true values.
In our experiments, we implemented HesScale using the \emph{BackPACK} framework \cite{dangel2019backpack}, which allows easy implementation of backpropagation of statistics other than the gradient.
To measure the approximation quality of the Hessian diagonals for different methods, we use the $L^1$ distance between the exact Hessian diagonals and their approximations, where $L^1(\mathbf{a}, \mathbf{b}) = \sum_i | a_i - b_i |$. Our task here is supervised classification, and data examples are sampled randomly from MNIST. We used a network of three hidden layers with \emph{tanh} activations, each containing $32$ units. The network weights and biases are initialized using Kaiming initialization \cite{he2015delving}. We trained the network with SGD using a batch size of $1$. For each example pair, we compute the exact Hessian diagonals for each layer and their approximations from each method. All layers' errors are summed and then averaged over $1000$ data examples for each method. In this experiment, we used 40 different initializations for the network weights, shown as colored dots in Fig.\ \ref{fig:approx_q_norm}. In this figure, we show the average error incurred by each method normalized by the average error incurred by HesScale. Any approximation that incurs an averaged error above 1 has a worse approximation than HesScale, and any approximation with an error less than 1 has a better approximation than HesScale. Moreover, we show the layer-wise error for each method in Fig.\ \ref{fig:approx_q_layerwise}.

\begin{figure}[ht]
\centering
\subfigure[Normalized $L^1$ error with respect to HesScale]{
     \centering
     \includegraphics[width=0.85\columnwidth]{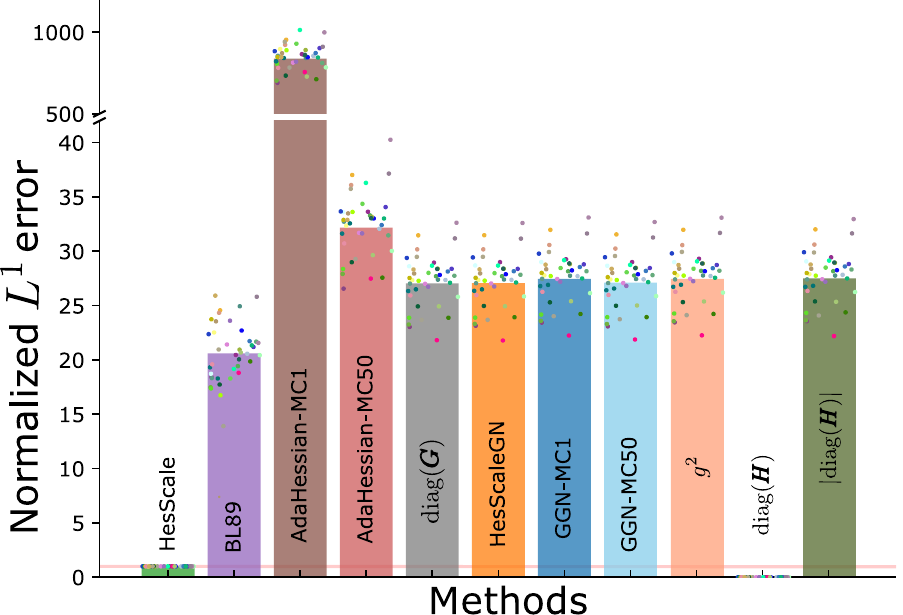}
     \label{fig:approx_q_norm}
}
\subfigure[Layer-wise $L^1$ error]{
     \includegraphics[width=0.85\columnwidth]{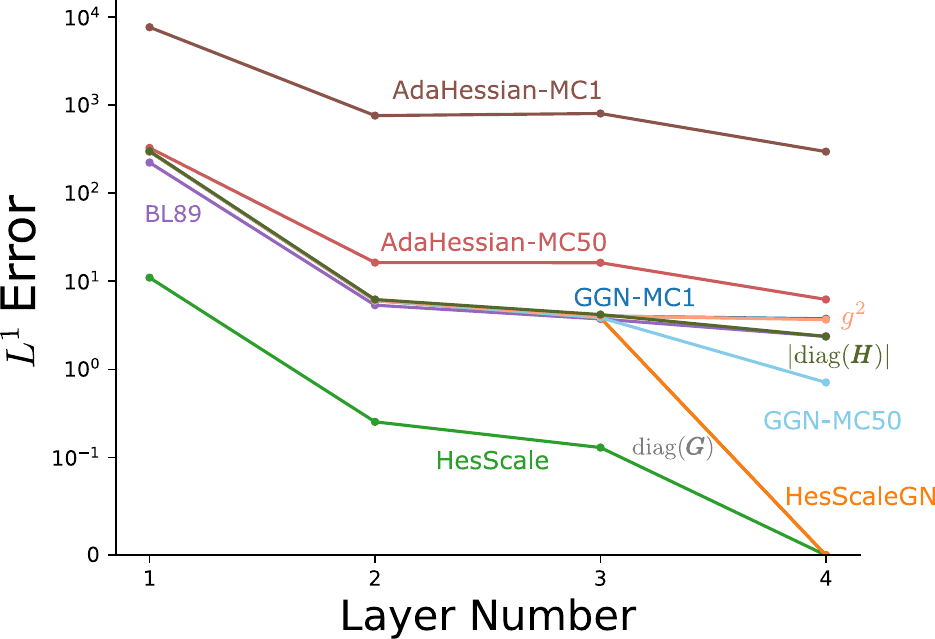}
     \label{fig:approx_q_layerwise}
}
\caption{(a) The average error for each method is normalized by the average error incurred by HesScale. Each colored point represents a different initialization. The norm of the vector of Hessian diagonals $|\operatorname{diag}(\mH)|$ is shown as a reference. (b) The average layer-wise error for each method is shown. HesScale(GN) has zero error at the last layer since it uses the exact entries there.}
\label{fig:approximation_quality}
\end{figure}

Different Hessian diagonal approximations are considered for comparison with HesScale. We included several deterministic and stochastic approximations for the Hessian diagonals. We also included the sample estimate of the Fisher Information Matrix done by squaring the gradients and denoted by $g^2$, which is adopted by many first-order methods (e.g., \citealp{kingma2014adam}). We compare HesScale with two stochastic approximation methods: AdaHessian \cite{yao2021adahessian}, and the Monte-Carlo (MC) estimate of the GGN matrix (GGN-MC, \citealp{dangel2019backpack}). We also compare HesScale with two deterministic approximation methods: the diagonals of the exact GGN matrix \cite{schraudolph2002fast} ($\operatorname{diag}(\mG)$) and the diagonal approximation by \citet{becker1988improving} (BL89). Since AdaHessian and GGN-MC are already diagonal approximations, we use them directly and show the error with $1$ MC sample (GGN-MC1 \& AdaHessian-MC1) and with $50$ MC samples (GGN-MC50 \& AdaHessian-MC50).

HesScale provides a better approximation than the other deterministic and stochastic methods. For stochastic methods, we use many MC samples to improve their approximation. However, their approximation quality is still poor. Methods approximating the GGN diagonals do not capture the complete Hessian information since the GGN and Hessian matrices are different when the activation functions are not piece-wise linear, as is the case for our tanh-activated network. Although these methods approximate the GGN diagonals, their approximation is significantly better than the AdaHessian approximation. 
Among the methods for approximating the GGN diagonals, HesScaleGN 
approximates the exact GGN diagonals closely.
This experiment clearly shows that HesScale achieves the best approximation quality (overall and across layers) compared to other stochastic and deterministic approximation methods. 

\subsection{Diagonality of Hessians w.r.t. Pre-activations}
To understand why HesScale gives such a high approximation quality, we investigate the structure of the matrices propagated in the full Hessian backpropagation. The high approximation quality suggests that those matrices are diagonally dominant; thus, dropping the off-diagonal elements does not significantly harm the approximation. Similar to the metric in \citet{balles2020geometry}, we introduce the metric $\rho \in [0,1]$ that measures the diagonal dominance of a matrix $\mA$ given by
\begin{align*}
    \rho(\mA) &= \frac{\| \operatorname{diag}(\mA) \|_\text{F}}{\| \mA \|_\text{F}},
\end{align*}
where $\| .\|_\text{F}$ is the Frobenius norm and $\operatorname{diag}(.)$ extracts the diagonal of a matrix. The metric $\rho$ gives a value of $1$ when $\mA$ is a diagonal matrix and a value of $0$ when $\mA$ is a hollow matrix. We use $\rho$ to to measure the diagonal dominance of $\{\nabla^2_{\va_i}\mathcal{L}\}_{i=1}^{L}$ (shown in Fig.\ \ref{fig:heatmaps}) on an MLP of four hidden layers each containing $128$ units. In Table \ref{tab:diag-dominance}, we show the diagonal dominance of those matrices before and after training. As a reference, we compute the diagonal dominance for a random matrix with standard Gaussian entries of the same size. We found that $\rho$ of a random matrix is $0.09$. The diagonal dominance metric for each matrix is averaged over $300$ independent runs. We trained for $10000$ iterations on EMNIST \cite{cohen2017emnist} with a batch size of $32$. Table \ref{tab:diag-dominance} shows that $\{\nabla^2_{\va_i}\mathcal{L}\}_{i=1}^{L}$ are diagonally dominant before and after training, which explains the high approximation quality of HesScale compared to other approximations.

\begin{figure}[ht]
\centering
\includegraphics[width=0.9\columnwidth]{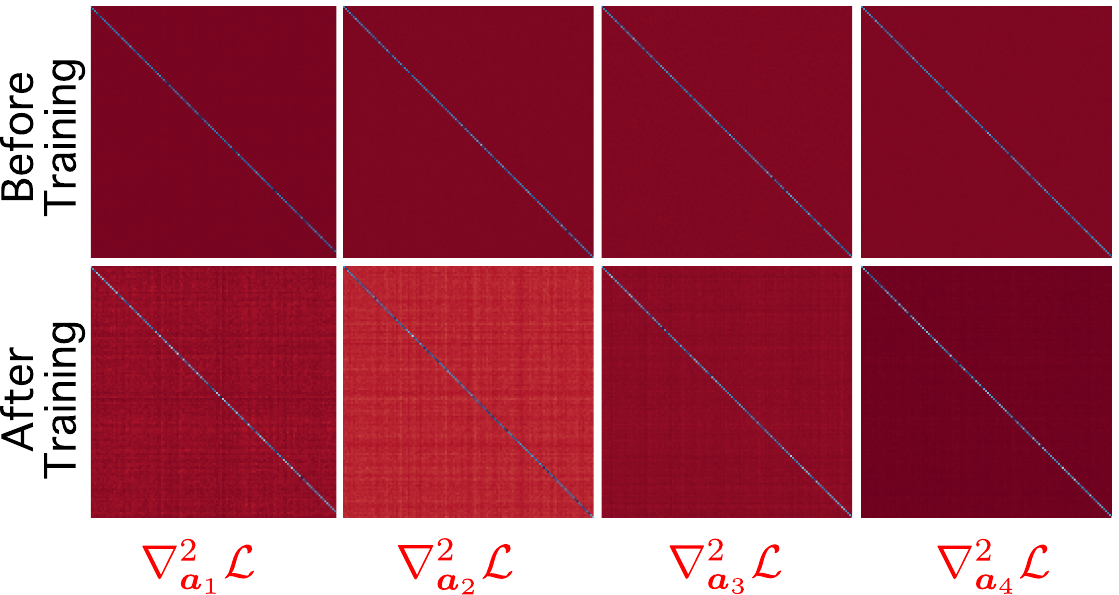}
\caption{Heat maps of Hessian of the loss w.r.t.\ pre-activations. $\{\nabla^2_{\va_i}\mathcal{L}\}_{i=1}^{4}$ visually appear diagonally dominant. Red represents a small magnitude, and blue represents a large magnitude.}
\label{fig:heatmaps}
\end{figure}

\begin{table}[ht]
\centering
\caption{Diagonal dominance before and after training}
\label{tab:diag-dominance}
\begin{tabular}{@{}ccc@{}}
\toprule
Layer Number & $\rho$ (before training) & $\rho$ (after training) \\ 
\midrule
1            & 0.94                    & 0.75            \\
2            & 0.91                    & 0.54            \\
3            & 0.89                    & 0.81            \\
4            & 0.91                    & 0.98            \\
\bottomrule
\end{tabular}
\end{table}

%%%%%%%%%%%%%%%%%%%%%%%%%%%%%%%%%%%
\vspace{-2mm}
\section{Two applications of HesScale}
We utilize HesScale in two applications of second-order information: optimization and step-size scaling.
\subsection{Second-order Optimization}
We introduce an efficient optimizer based on HesScale, which we call \textit{AdaHesScale} (Algorithm \ref{alg:adahesscale}). We follow the same style introduced in Adam \cite{kingma2014adam} in using the squared diagonal approximation instead of the squared gradients to update the moving average.
Moreover, we introduce another optimizer based on HesScaleGN, which we call \textit{AdaHesScaleGN}. 
For the convergence proof of methods with Hessian diagonals, we refer the reader to Appendix \ref{appendix:convergence-proof}. In addition, we provide a scalability comparison for AdaHesScale and AdaHesScaleGN against other optimization methods in Appendix \ref{appendix:scalability}.

\begin{algorithm}[ht]
\caption{AdaHesScale}\label{alg:adahesscale}
\begin{algorithmic}
\STATE {\bfseries Require:} Neural network $f$ with weights $\{\mW_1,...,\mW_L\}$ and a dataset $\mathcal{D}$
\STATE {\bfseries Require:} Small number $\epsilon \leftarrow 10^{-8}$
\STATE {\bfseries Require:} Exponential decay rates $\beta_1,\beta_2\in [0,1)$
\STATE {\bfseries Require:} step size $\alpha$
\STATE {\bfseries Require:} Initialize $\{\mW_1,...,\mW_L\}$
\STATE Initialize time step $t\leftarrow0$.
\FOR{$l$ in $\{L,L-1,...,1\}$} 
\STATE $\mM_{l} \leftarrow 0; \quad \mV_{l} \leftarrow 0$
\ENDFOR
\FOR{$(\vx, y)$ in $\mathcal{D}$}
\STATE $t\leftarrow t+1$
\STATE $\vr_{L+1} \leftarrow \vs_{L+1}\leftarrow \bm{\emptyset}$ 
\FOR{$l$ in $\{L,L-1,...,1\}$}
\STATE Compute Loss $\mathcal{L}(\vx, y)$
\STATE $\mF_{l},\mS_{l}, \vr_{l}, \vs_{l}\leftarrow$\texttt{HesScale}($\mathcal{L}, l, \vr_{l+1}, \vs_{l+1}$)
\STATE $\mM_{l} \leftarrow \beta_1\mM_{l} + (1-\beta_1)\mF_l $ 
\STATE $\mV_{l} \leftarrow \beta_2\mV_{l} + (1-\beta_2)\mS_l^2 $  
\STATE $\hat{\mM}_{l} \leftarrow \mM_{l}/(1-\beta_1^{t})$ 
\STATE $\hat{\mV}_{l} \leftarrow \mV_{l}/(1-\beta_2^{t}) $ 
\STATE $\mW_{l} \leftarrow \mW_{l} - \alpha \hat{\mM}_{l}\oslash \left(\sqrt{\hat{\mV}_l} + \epsilon\right)$
\ENDFOR
\ENDFOR
\end{algorithmic}
\end{algorithm}

\subsection{Step-size Scaling for Robustness and Stability}
Scaling the step size has been a well-known approach for improving the robustness of supervised learning (e.g., \citealp{mahmood2012tuning}) and reinforcement learning (e.g., \citealp{dabney2012adaptive}). K-FAC \cite{martens2015optimizing} implemented a step-size scaling procedure based on trust region analysis, which was later studied with Adam \cite{clarke2023adam}. Such a method has been combined with some RL algorithms, giving algorithms such as ACKTR \cite{wu2017scalable}. The step-size mechanism scales the step size down when the update to be applied becomes outside the trust region radius, making the optimizer less sensitive to the choice of step size and, therefore, improving robustness. More recently, \citet{dohare2023overcoming} demonstrated the problem of extreme instability in deep RL methods when trained for extended periods. Such a problem could be mitigated if the optimizer becomes aware of the size of the update it makes, which is the quantity measured by those step-size scaling methods. Thus, we developed a step-size scaling method following the mechanism outlined in K-FAC and based on our HesScale approximation. The K-FAC mechanism scales down the step size by $\min \left(\alpha_\text{max}, \sqrt{\frac{2\Delta}{\vu^\top \mH \vu}} \right)$, where $\alpha_\text{max}$ is the maximum step size, $\Delta$ is the trust-region radius, and $\vu$ is the regular update if no scaling is performed. Algorithm \ref{alg:scaled-adahesscale} shows our step-size scaling applied to AdaHesScale using a HesScale approximation of $\mH$. 

\begin{algorithm}[ht]
\caption{AdaHesScale with {\color{blue}step-size scaling}}\label{alg:scaled-adahesscale}
\begin{algorithmic}
\STATE {\bfseries Require:} Neural network $f$ with weights $\{\mW_1,...,\mW_L\}$ and a dataset $\mathcal{D}$
\STATE {\bfseries Require:} Small number $\epsilon \leftarrow 10^{-8}$
\STATE {\bfseries Require:} Exponential decay rates $\beta_1,\beta_2\in [0,1)$
\STATE {\bfseries Require:} step size $\alpha$, {\color{blue}trust-region radius $\Delta$}
\STATE {\bfseries Require:} Initialize $\{\mW_1,...,\mW_L\}$
\STATE Initialize time step $t\leftarrow0$.
\FOR{$l$ in $\{L,L-1,...,1\}$} 
\STATE $\mM_{l} \leftarrow 0; \quad \mV_{l} \leftarrow 0; \quad \mU_{l} \leftarrow 0$
\ENDFOR
\FOR{$(\vx, y)$ in $\mathcal{D}$}
\STATE $t\leftarrow t+1$
\STATE {\color{blue}$h \leftarrow 0$}
\STATE $\vr_{L+1} \leftarrow \vs_{L+1}\leftarrow \bm{\emptyset}$ 
\FOR{$l$ in $\{L,L-1,...,1\}$}
\STATE Compute Loss $\mathcal{L}(\vx, y)$
\STATE $\mF_{l},\mS_{l}, \vr_{l}, \vs_{l}\leftarrow$\texttt{HesScale}($\mathcal{L}, l, \vr_{l+1}, \vs_{l+1}$)
\STATE $\mM_{l} \leftarrow \beta_1\mM_{l} + (1-\beta_1)\mF_l $ 
\STATE $\mV_{l} \leftarrow \beta_2\mV_{l} + (1-\beta_2)\mS_l^2 $  
\STATE $\hat{\mM}_{l} \leftarrow \mM_{l}/(1-\beta_1^{t})$ 
\STATE $\hat{\mV}_{l} \leftarrow \mV_{l}/(1-\beta_2^{t}) $ 
\STATE $\mU_{l} \leftarrow \alpha \hat{\mM}_{l}\oslash \left(\sqrt{\hat{\mV}_l} + \epsilon\right)$
\STATE {\color{blue}$\mS \leftarrow \sqrt{\hat{\mV}_l} \circ \mU_l^2$}
\STATE {\color{blue} $h \leftarrow h + \bm1^\top \mS \bm1$}
\ENDFOR
\STATE {\color{blue} $\eta \leftarrow \text{min}\left(1, \sqrt{\frac{2\Delta}{h}}\right)$}
\FOR{$l$ in $\{L,L-1,...,1\}$}
\STATE {\color{blue}$\mW_l \leftarrow \mW_l - \eta \mU_l$}
\ENDFOR
\ENDFOR
\end{algorithmic}
\end{algorithm}

%%%%%%%%%%%%%%%%%%%%%%%%%%%%%%%%%%%

\section{Supervised Learning Experiments}
We compare the performance of our optimizers---AdaHesScale and AdaHesScaleGN---with three second-order optimizers: BL89, GGNMC, and AdaHessian.
We also include comparisons with two first-order methods: Adam and SGD. We exclude K-FAC from our comparisons due to its relatively high cost.

Our optimizers are evaluated in the supervised setting with two experiments using different architectures on the CIFAR-100 dataset.
Instead of attempting to achieve state-of-the-art performance with specialized techniques and architectures, we followed the DeepOBS benchmarking work \cite{schneider2019deepobs} and compared the optimizers in their generic form using relatively simple networks to verify the validity of our method.

In the first experiment, we used the CIFAR-100 3C-3D task from DeepOBS.
The network consists of three convolutional layers with \emph{ReLU} activations, each followed by max pooling. After that, two fully connected layers ($512$ and $256$ units) with \emph{ReLU} activations are used.
We use \textit{ELU} instead of \textit{ReLU}, which is used in DeepOBS, to differentiate between the performance of AdaHesScale and AdaHesScaleGN.
We train each method for $200$ epochs with a batch size of $128$.
In the second experiment, we use the CIFAR-100 ALL-CNN task from DeepOBS with the ALL-CNN-C network, which consists of 9 convolutional layers \cite{springenberg2014striving} with \emph{ReLU} activations. Again, we use \emph{ELU} instead of \emph{ReLU}, which is used in DeepOBS, to differentiate between the performance of AdaHesScale and AdaHesScaleGN. We train each method for $350$ epochs with a batch size of $256$. 

\begin{figure}[ht]
\centering
\subfigure[CIFAR-100 3C3D]{
     \includegraphics[width=0.93\columnwidth]{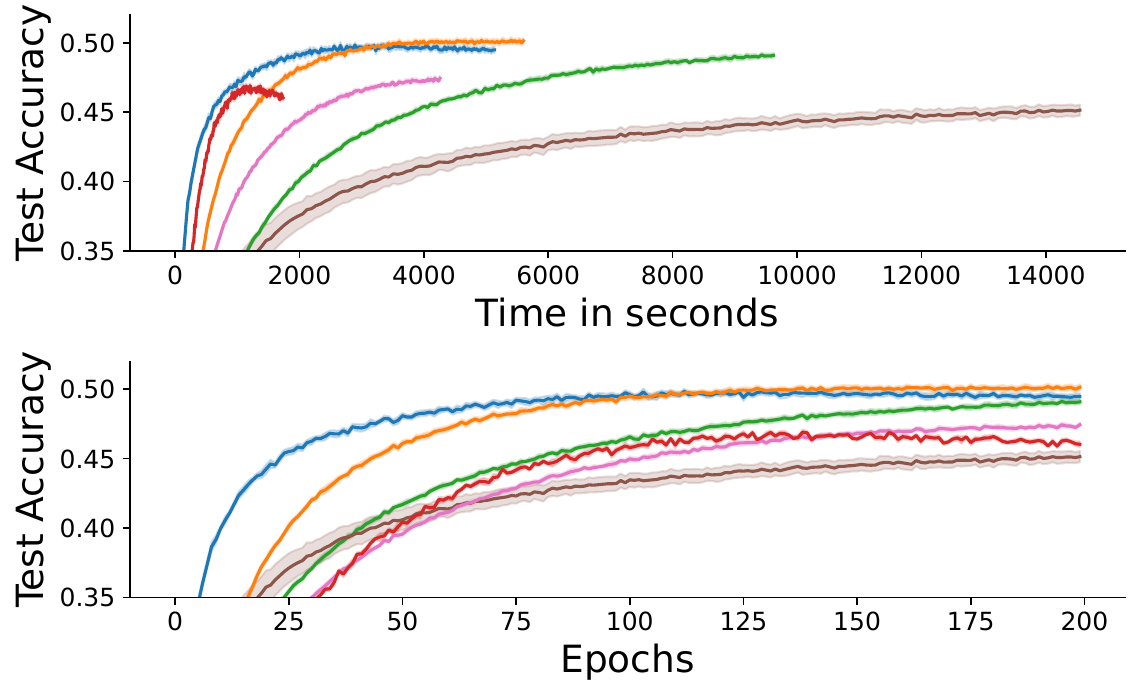}
     \label{fig:optimization_CIFAR100_3c3d_time}
}
\subfigure[CIFAR-100 All-CNN-C]{
     \includegraphics[width=0.93\columnwidth]{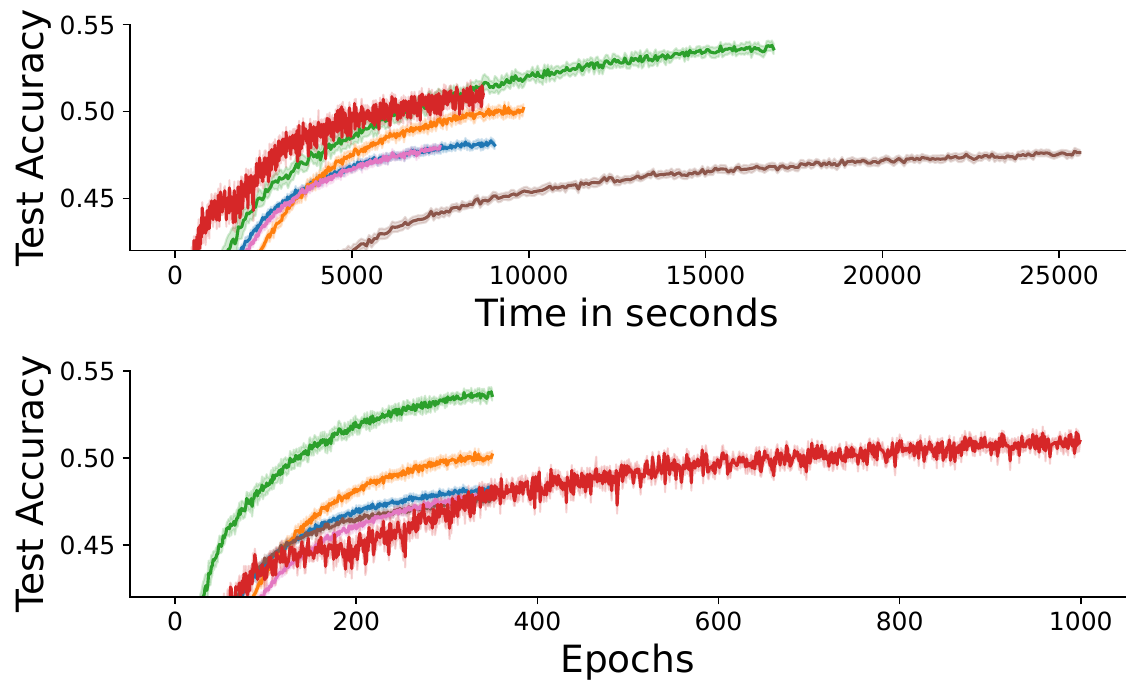}
     \label{fig:optimization_CIFAR100_CNN_time}
}
\includegraphics[width=0.9\columnwidth]{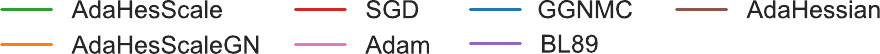}
\vspace{-0.2cm}
\caption{CIFAR-100 3C3D and CIFAR-100 ALL-CNN classification tasks. (top) We show the time taken by each algorithm in seconds, and (bottom) we show the learning curves in the number of epochs. The shaded area represents the standard error. BL89 achieves lower than $0.35$ and is not visible.}
\label{fig:optimization_CIFAR100_time}
\end{figure}

Fig.\ \ref{fig:optimization_CIFAR100_3c3d_time} and Fig.\ \ref{fig:optimization_CIFAR100_CNN_time} show the results on CIFAR-100 ALL-CNN and CIFAR-100 3C3D tasks against the number of epochs and against time in seconds. In both experiments, we used $\beta_1=0.9$ and $\beta_2=0.999$ for all adaptive methods. We run SGD for the time that matches the maximum time taken by any optimizer to have a fair comparison since it requires a small cost relative to other methods (Fig.\ \ref{fig:optimization_CIFAR100_CNN_time}) unless the performance starts to go down within the predetermined number of epochs in the experiment (Fig.\ \ref{fig:optimization_CIFAR100_3c3d_time}).

The performance of each method is averaged over $30$ independent runs. Each independent run has the same initialization across all algorithms. We performed a hyperparameter search for each method to find the best step size. Using each method's best step size on the validation set, we show the performance of the method against the time in seconds needed to complete the required number of epochs, which better depicts the computational efficiency of the methods.

In both CIFAR-100 3C3D and CIFAR-100 ALL-CNN, we notice that AdaHessian performed worse than all methods except BL89. This result is aligned with AdaHessian's inability to accurately approximate the Hessian diagonals, as shown in Fig.\ \ref{fig:approximation_quality}. Moreover, AdaHessian required more time than all other methods. While being time-efficient, AdaHesScaleGN consistently outperformed all methods in CIFAR-100 3C3D, and it outperformed most methods except AdaHesScale in CIFAR-100 ALL-CNN. Our experiments indicate that incorporating HesScale and HesScaleGN approximations in optimization methods can be of significant performance advantage in both computation and accuracy. AdaHesScale and AdaHesScaleGN outperformed other optimizers, likely due to their accurate approximation of the diagonals of the Hessian and GGN, respectively. We provide the sensitivity analysis in Appendix \ref{appendix:additional-classification}.

%%%%%%%%%%%%%%%%%%%%%%%%%%%%%%%%%%%

\section{Reinforcement Learning Experiments}
We investigate the performance of AdaHesScale against other optimizers when used with two reinforcement learning algorithms, A2C \cite{mnih2016asynchronous} and PPO \cite{schulman2017proximal}, on the MuJoCo environments \cite{todorov2012mujoco}. We exclude optimizers based on GGN and GGNMC since their BackPACK implementation is limited to classification and regression. Then, we investigate the effect of step-size scaling on robustness and stability with AdaHesScale and Adam. In this section's experiments, we use MLPs of two hidden layers each containing $64$ units with \emph{tanh} activations, similar to the architectures used in CleanRL \cite{huang2022cleanrl}.

\begin{figure*}
\centering
\includegraphics[width=\textwidth]{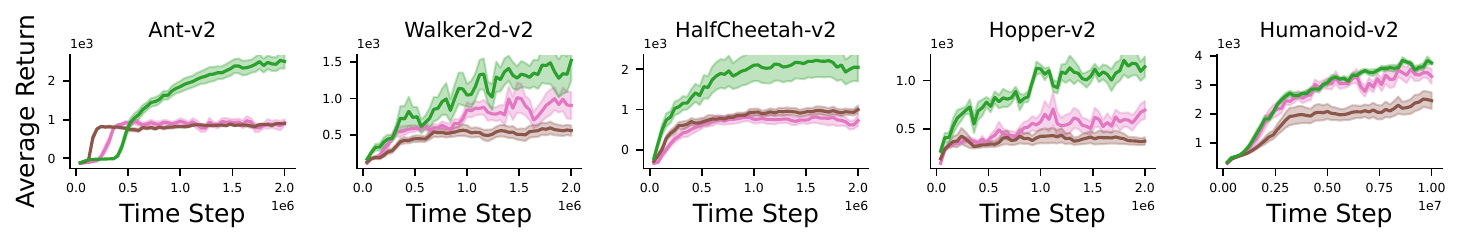}
\includegraphics[width=\textwidth]{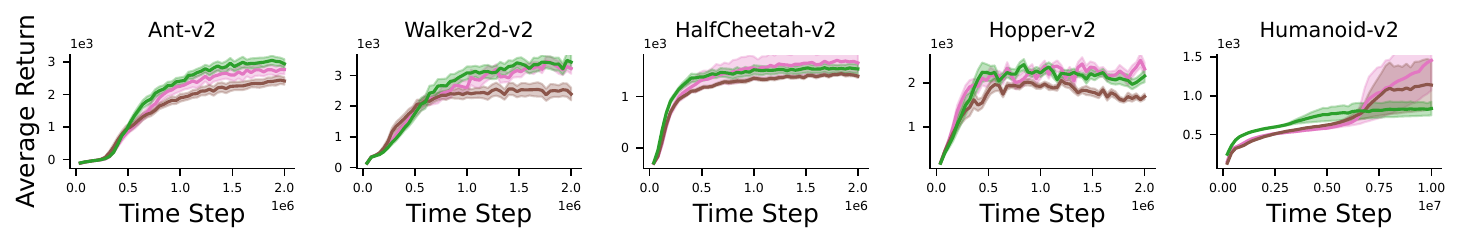}
\includegraphics[width=0.4\textwidth]{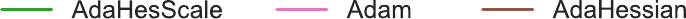}
\caption{Performance of A2C (first row) and PPO (second row) with AdaHesScale, Adam, and AdaHessian on $5$ MuJoCo environments. We show the undiscounted return averaged over $10$ independent runs. The shaded area represents the standard error.}
\label{fig:rl-ppo-a2c}
\end{figure*}

\subsection{Performance}
We start by focusing on a performance-based comparison. Fig.\ \ref{fig:rl-ppo-a2c} shows the performance of A2C/PPO with AdaHesScale against A2C/PPO with Adam and AdaHessian on five environments. We add results for five other MuJoCo environments to Appendix \ref{appendix:rl-additional}. In A2C, we observe that AdaHesScale significantly boosts performance on Ant, Walker2d, HalfCheetah, and Hopper compared to Adam and AdaHessian. In PPO, AdaHesScale performs similarly to Adam and outperforms AdaHessian in most environments.

\subsection{Robustness and Stability}
Next, we study the effect of step-size scaling using the HesScale approximation on robustness and stability with AdaHesScale, which we call \emph{Scaled AdaHesScale}, and Adam, which we call \emph{Scaled Adam}. We used a trust-region radius $\Delta = 10^{-8}$ and applied the step-size scaling mechanism on both the actor and the critic networks. Fig.\ \ref{fig:rl-robustness} shows the robustness of the methods for the step-size choice. For each environment, we normalized the average return between the minimum and maximum values across both optimizers to have the same range $[0,1]$. We observe that our step-size scaling makes both scaled optimizers insensitive to the step-size choice (compared to Adam) with PPO in all MuJoCo environments except for Humanoid, which does not benefit from the step-size scaling. We also redo the same experiment with the A2C algorithm in Appendix \ref{appendix:rl-additional}.

\begin{figure}[t]
\centering
\includegraphics[width=\columnwidth]{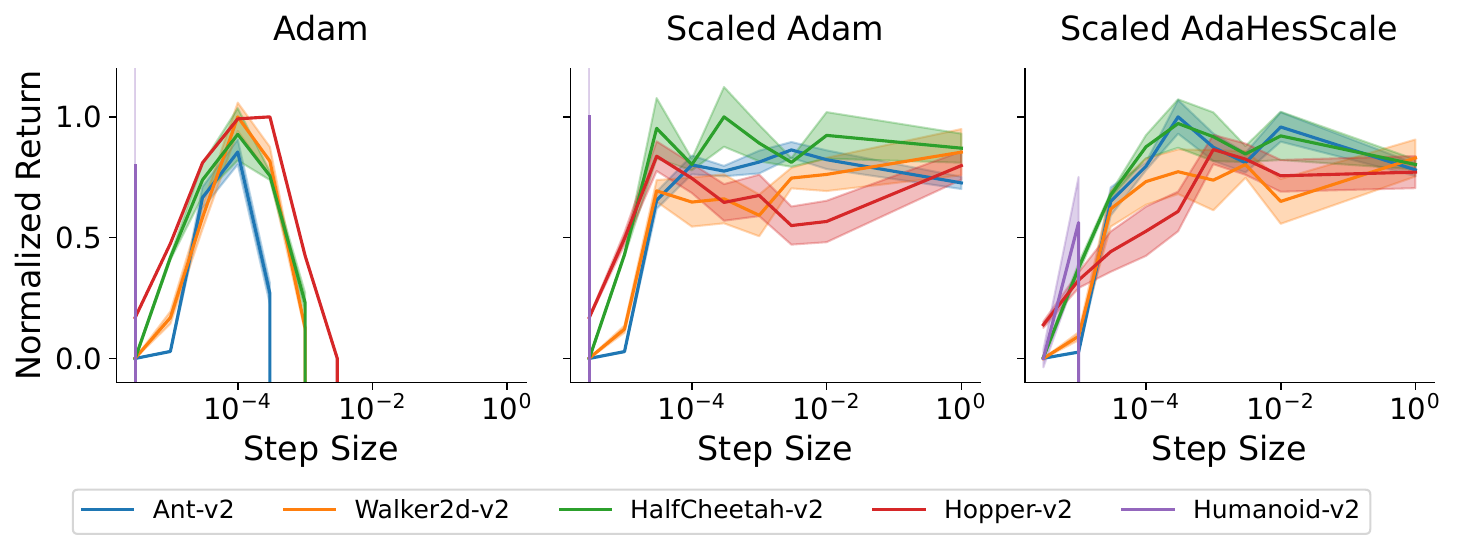}
\vspace{-0.5cm}
\caption{Robustness of HesScale-based step-size scaling with AdaHesScale and Adam on $5$ MuJoCo environments. We show the undiscounted return averaged over $10$ independent runs. The shaded area represents the standard error.}
\label{fig:rl-robustness}
\end{figure}

Finally, we examine the effectiveness of step-size scaling in stability using the \emph{UR-Reacher}, a real-robot task introduced by \cite{mahmood2018benchmarking, mahmood2018setting} to benchmark RL algorithms on real robots, which has since been used or adopted in a number of works \cite{lucchi2020robo, lan2021model, yuan2022asynchronous, farrahi2023reducing, che2023correcting, wang2023real}. The goal is to move the end effector of the UR5 robotic arm in the 2d space and reach some specified point. In challenging environments, the default step size can lead to large updates, causing the policy to deteriorate with time \cite{dohare2023overcoming}. However, the step-size scaling mechanism could mitigate this issue, which lowers the step size when the update is too big. Fig.\ \ref{fig:rl-robotics} shows five runs for Adam and Scaled Adam each, where the learned policy by Adam deteriorates for many runs after reaching a maximum performance at around $80$K steps, supporting the observation by \citet{dohare2023overcoming}. On the other hand, step-size scaling maintains consistent performance for Scaled Adam for all five runs, demonstrating its effectiveness in addressing the instability issue.

\vspace{-0.3cm}
\begin{figure}[ht]
\centering
\includegraphics[width=\columnwidth]{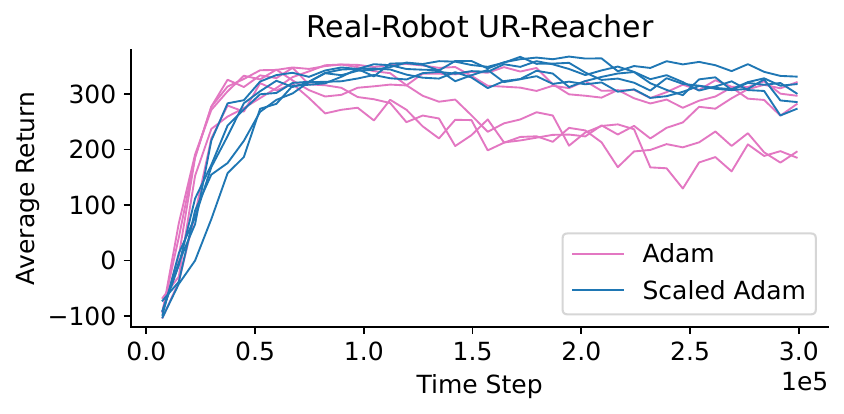}
\vspace{-0.5cm}
\caption{Performance of HesScale-based step-size scaling with Adam against standard Adam on UR-Reacher robotics task. We show the undiscounted return of five different independent runs.}
\label{fig:rl-robotics}
\end{figure}

%%%%%%%%%%%%%%%%%%%%%%%%%%%%%%%%%%%

\section{Conclusion}
HesScale is a scalable and efficient second-order method for approximating the diagonals of the Hessian at every network layer.
Our results showed that our methods provide a more accurate approximation for the Hessian diagonals while requiring small additional computations. 
We demonstrated how HesScale can be used to build efficient second-order optimization methods for both supervised reinforcement learning. We also showed how HesScale can be used to build step-size scaling mechanisms to help improve the stability and robustness of reinforcement learning methods.

%%%%%%%%%%%%%%%%%%%%%%%%%%%%%%%%%%%

\section*{Acknowlegement}
We gratefully acknowledge funding from the Canada CIFAR AI Chairs program, the Reinforcement Learning and Artificial Intelligence (RLAI) laboratory, the Alberta Machine Intelligence Institute (Amii), and the Natural Sciences and Engineering Research Council (NSERC) of Canada. We would also like to thank the Digital Research Alliance of Canada for providing the computational resources needed. Resources used in preparing this research were provided, in part, by the Province of Ontario, the Government of Canada through CIFAR, and companies sponsoring Vector Institute.

%%%%%%%%%%%%%%%%%%%%%%%%%%%%%%%%%%%

\section*{Impact Statement}
This paper presents work whose goal is to advance the field of Machine Learning. There are many potential societal consequences of our work, none of which we feel must be specifically highlighted here.

%%%%%%%%%%%%%%%%%%%%%%%%%%%%%%%%%%%

\bibliographystyle{apalike}
\bibliography{references.bib}

%%%%%%%%%%% APPENDIX %%%%%%%%%%%%%%

\clearpage
\appendix
\onecolumn

\section{Convergence proof of AdaHesScale and AdaHesScaleGN}
\label{appendix:convergence-proof}
In this section, we provide a convergence proof for AdaHesScale and AdaHesScaleGN that resembles the proof given by \citet{zaheer2018adaptive} for the convergence of adaptive optimization methods (e.g., Adam) in nonconvex problems. The following proof shows the convergence to a stationary point up to the statistical limit of the variance of the gradients, where $\| \nabla f(\boldsymbol\theta) \|^2 \leq \delta$ represent a $\delta$-accurate solution and is used to measure the stationarity of $\boldsymbol\theta$. Nonconvex optimization problems can be written as:
\begin{align*}
\min_{\boldsymbol\theta \in \mathbb{R}^d} f (\boldsymbol\theta) &\doteq \mathbb{E}_{S \sim P} \left[ \mathcal{L}(\boldsymbol\theta, S)\right]
\end{align*}
where $f$ is the expected loss, $\mathcal{L}$ is the sample loss, $S$ is a random variable for samples and $\boldsymbol\theta$ is a vector of weights parametrizing $\mathcal{L}$. We assume that $\mathcal{L}$ is $L$-smooth, meaning that there exist a constant $L$ that satisfy
\begin{align}
\| \nabla \mathcal{L}(\boldsymbol\theta_1, s) - \nabla \mathcal{L}(\boldsymbol\theta_2, s) \| \leq L \| \boldsymbol\theta_2 - \boldsymbol\theta_1 \|, \; \forall \boldsymbol\theta_1, \boldsymbol\theta_2 \in \mathbb{R}^d, s \in \mathcal{S}.
\end{align}

Similar to the proof by \citet{zaheer2018adaptive}, we further assume that $\mathcal{L}$ has bounded gradients $|\nabla[\mathcal{L}(\boldsymbol\theta, s)]_i|\leq G, \forall i \in \mathbb{R}^d, s\in \mathcal{S}$ and bounded variance in the gradients $\mathbb{E}\| \nabla\mathcal{L}(\boldsymbol\theta, S) - \nabla f(\boldsymbol\theta)  \|^2 \leq \sigma^2, \forall \boldsymbol\theta \in \mathbb{R}^d$. Note that the assumption of L-smoothness on the sample loss result in L-smooth expected loss too, which is given by $\| \nabla f(\boldsymbol\theta_1) -\nabla f(\boldsymbol\theta_2) \| \leq L \| \boldsymbol\theta_1 -\boldsymbol\theta_2 \|$.

Remember that the update rule of AdaHesScale and AdaHesScale-GN can be written as follows when the parameters are stacked in a single vector $\boldsymbol\theta$:
\begin{align*}
\theta_{t+1, i} = \theta_{t, i} - \alpha \frac{m_{t,i}}{\sqrt{v_{t,i}} + \epsilon},
\end{align*}
where $\vm_t$ and $\vv_t$ are exponential moving averages of the first derivatives $\vg_t$ and second derivatives $\vh_t$, respectively. Similar to \citet{zaheer2018adaptive}, we write the proof for $\beta_1=0$ making $m_{t,i} = g_{t,i}$ for the sake of simplicity. However, this proof should extend to the general case.

Since the expected loss $f$ is $L$-smooth, we can write the following:
\begin{align}
f(\boldsymbol\theta_{t+1} ) &\leq  f(\boldsymbol\theta_{t}) + (\nabla f(\boldsymbol\theta_t))^{\top} (\boldsymbol\theta_{t+1} - \boldsymbol\theta_t) + \frac{L}{2} \| \boldsymbol\theta_{t+1} - \boldsymbol\theta_t \|_2^2\\
&= f(\boldsymbol\theta_t) - \alpha \sum_{i=1}^{d} \left( \nabla[f(\boldsymbol\theta_t)]_i \frac{g_{t,i}}{\sqrt{v_{t,i}}+\epsilon} \right) + \frac{L \alpha^2}{2} \sum_{i=1}^{d} \frac{g_{t,i}^2}{(\sqrt{v_{t,i}}+\epsilon)^2}.
\end{align}

Next, we take the conditional expectation of $f(\boldsymbol\theta_{t+1})$ as follows:
\begin{align*}
\mathbb{E}_t [f(\boldsymbol\theta_{t+1}) | \boldsymbol\theta_t] &\leq f(\boldsymbol\theta_t) - \alpha \sum_{i=1}^{d} \left( [\nabla f(\boldsymbol\theta_t)]_i \mathbb{E}_t\left[\frac{g_{t,i}}{\sqrt{v_{t,i}}+\epsilon}  \right] \right)+ \frac{L \alpha^2}{2} \sum_{i=1}^{d}\mathbb{E}_t \left[\frac{g_{t,i}^2}{(\sqrt{v_{t,i}}+\epsilon)^2} \right]\\
&= f(\boldsymbol\theta_t) - \alpha \sum_{i=1}^{d} \left( [\nabla f(\boldsymbol\theta_t)]_i \mathbb{E}_t\left[\frac{g_{t,i}}{ \sqrt{v_{t,i}} +\epsilon} - \frac{g_{t,i}}{\epsilon} + \frac{g_{t,i}}{\epsilon}  \right] \right)+ \frac{L \alpha^2}{2} \sum_{i=1}^{d}\mathbb{E}_t \left[\frac{g_{t,i}^2}{(\sqrt{v_{t,i}}+\epsilon)^2} \right]\\
&= f(\boldsymbol\theta_t) - \alpha \sum_{i=1}^{d} \left( [\nabla f(\boldsymbol\theta_t)]_i \mathbb{E}_t\left[\frac{g_{t,i}}{ \sqrt{v_{t,i}} +\epsilon} - \frac{g_{t,i}}{\epsilon} \right] \right) - \alpha \sum_{i=1}^d \frac{[\nabla f(\boldsymbol\theta_t)]_i^2}{\epsilon} + \frac{L \alpha^2}{2} \sum_{i=1}^{d}\mathbb{E}_t \left[\frac{g_{t,i}^2}{(\sqrt{v_{t,i}}+\epsilon)^2} \right]\\
&\leq f(\boldsymbol\theta_t) + \alpha \left|\sum_{i=1}^{d}  [\nabla f(\boldsymbol\theta_t)]_i \mathbb{E}_t\left[\frac{g_{t,i}}{ \sqrt{v_{t,i}} +\epsilon} - \frac{g_{t,i}}{\epsilon} \right] \right| - \alpha \sum_{i=1}^d \frac{[\nabla f(\boldsymbol\theta_t)]_i^2}{\epsilon} + \frac{L \alpha^2}{2\epsilon^2} \sum_{i=1}^{d}\mathbb{E}_t \left[g_{t,i}^2 \right]\\
&\leq f(\boldsymbol\theta_t) + \alpha \sum_{i=1}^{d}  \left|\nabla[f(\boldsymbol\theta_t)]_i\right| \mathbb{E}_t\Bigg[\underbrace{\left| \frac{g_{t,i}}{ \sqrt{v_{t,i}} +\epsilon} - \frac{g_{t,i}}{\epsilon}  \right|}_{T} \Bigg] - \alpha \sum_{i=1}^d \frac{[\nabla f(\boldsymbol\theta_t)]_i^2}{\epsilon} + \frac{L \alpha^2}{2\epsilon^2} \sum_{i=1}^{d}\mathbb{E}_t \left[g_{t,i}^2 \right]
\end{align*}
Note that we used $\mathbb{E}_t[g_{t,i}] = [\nabla f(\boldsymbol\theta_t)]_i$ in the third line, and $\mathbb{E}[X] \leq \mathbb{E}[|X|]$ and $|\sum_i^d x_i| \leq \sum_{i=1}^d |x_i|$ in the fifth line. We also used $\frac{1}{(a+b)^2} \leq \frac{1}{b^2}, \forall a, b > 0$ for the last term in the fourth line.

Now, we can further bound the term $T$ as follows:
\begin{align*}
T &= \left|  \frac{g_{t,i}}{ \sqrt{v_{t,i}} +\epsilon} - \frac{g_{t,i}}{\epsilon}  \right|\\
&\leq \left|  \frac{g_{t,i}}{ \sqrt{v_{t,i}} +\epsilon}\right| + \left| \frac{g_{t,i}}{\epsilon}  \right|\\
&\leq  \frac{2|g_{t,i}|}{\epsilon}.
\end{align*}
Note that we used $\frac{1}{a+b} \leq \frac{1}{b}, \forall a, b > 0$. Next, let us go back to the main inequality and further bound it as follows:
\begin{align*}
\mathbb{E}_t [f(\boldsymbol\theta_{t+1}) | \boldsymbol\theta_t] 
&\leq f(\boldsymbol\theta_t) + \frac{2\alpha}{\epsilon} \sum_{i=1}^{d} \left|\nabla[f(\boldsymbol\theta_t)]_i\right| \mathbb{E}_t\left[|g_{t,i}|\right] - \frac{\alpha}{\epsilon} \sum_{i=1}^d [\nabla f(\boldsymbol\theta_t)]_i^2 +  \frac{L \alpha^2}{2\epsilon^2} \sum_{i=1}^{d}\mathbb{E}_t \left[g_{t,i}^2 \right] \\
&\leq f(\boldsymbol\theta_t) + \frac{2\alpha G^2}{\epsilon} - \frac{\alpha}{\epsilon} \sum_{i=1}^d [\nabla f(\boldsymbol\theta_t)]_i^2 +  \frac{L \alpha^2}{2\epsilon^2} \sum_{i=1}^{d}\mathbb{E}_t \left[g_{t,i}^2 \right].
\end{align*}
We used in the second inequality the bounded sample-gradient assumption, $|g_{t_i}|\leq G$. Moreover, we can have the same bound for the expected gradient as follows:
\begin{align*}
    \|\nabla f(\boldsymbol\theta)\| &= \|\mathbb{E}_S[\nabla\mathcal{L}(\boldsymbol\theta, S)]\| \leq \mathbb{E}_S[\|\nabla\mathcal{L}(\boldsymbol\theta, S)\|] \leq G,
\end{align*}
which leads to bounded expected gradients in each dimension, $|\nabla[f(\boldsymbol\theta)]_i|\leq G$.

From the bounded variance assumption, we know that the $\mathbb{E}[\|\vg_{t}\|^2]$ is bounded as follows:
\begin{align*}
\mathbb{E}[\|\vg_{t}\|^2] &\leq \frac{\sigma^2}{b_t} + \|\nabla f(\boldsymbol\theta_t)\|^2.
\end{align*}

We can further bound $\mathbb{E}_t [f(\boldsymbol\theta_{t+1}) | \boldsymbol\theta_t]$ as follows:
\begin{align*}
\mathbb{E}_t [f(\boldsymbol\theta_{t+1}) | \boldsymbol\theta_t] &\leq f(\boldsymbol\theta_t) + \frac{2\alpha G^2}{\epsilon} - \frac{\alpha}{\epsilon} \sum_{i=1}^d [\nabla f(\boldsymbol\theta_t)]_i^2 +  \frac{L \alpha^2}{2\epsilon^2} \left(\frac{\sigma^2}{b_t} + \|\nabla f(\boldsymbol\theta_t)\|^2\right) \\
&= f(\boldsymbol\theta_t) + \frac{2\alpha G^2}{\epsilon} - \frac{\alpha}{\epsilon} \|\nabla f(\boldsymbol\theta_t)\|^2 +  \frac{L \alpha^2}{2\epsilon^2} \left(\frac{\sigma^2}{b_t} + \|\nabla f(\boldsymbol\theta_t)\|^2\right) \\
&= f(\boldsymbol\theta_t) - \|\nabla f(\boldsymbol\theta_t)\|^2 \left(\frac{\alpha}{\epsilon}-\frac{L \alpha^2}{2\epsilon^2}\right) + \frac{L\alpha^2 \sigma^2 + 4\alpha \epsilon b_t G^2}{2b\epsilon^2}.
\end{align*}

Rearranging the inequality, taking expectations on both sides, and using the telescopic sum, we can write the following:
\begin{align*}
\left(\frac{2\epsilon\alpha - L\alpha^2}{2\epsilon^2}\right) \sum_{t=1}^{T} \mathbb{E} \|\nabla f(\boldsymbol\theta_t)\|^2  &\leq f(\boldsymbol\theta_1) - \mathbb{E} [f(\boldsymbol\theta_{T+1})] + \frac{T (L\alpha^2 \sigma^2 + 4\alpha\epsilon b_t G^2)}{2b\epsilon^2}.
\end{align*}
Multiplying both sides by $\frac{2\epsilon^2}{T(2\epsilon\alpha - L\alpha^2)}$ and using the fact that $f$ is the lowest at the global minimum $\boldsymbol\theta^*$: $f(\boldsymbol\theta_{T+1}) \geq f(\boldsymbol\theta^*)$ as follows:
\begin{align*}
\frac{1}{T}\sum_{t=1}^{T} \mathbb{E} \|\nabla f(\boldsymbol\theta_t)\|^2  &\leq 2\epsilon^2\frac{f(\boldsymbol\theta_1) - f(\boldsymbol\theta^*)}{T(2\epsilon\alpha - L\alpha^2)} +  \frac{2\epsilon^2 (L\alpha \sigma^2 + 4 b_t G^2)}{b_t (2\epsilon - L\alpha)},
\end{align*}
which shows the algorithm converges to a stationary point. However, in the limit $T\rightarrow \infty$, the algorithm has to have an increasing batch size similarly to \citet{zaheer2018adaptive}.

%%%%%%%%%%%%%%%%%%%%%%%%%%%%%%%%%%%

\section{Hessian diagonals of the log-likelihood function for two common distributions}
\label{appendix:linear-likelihood-hessian}
Here, we provide the diagonals of the Hessian matrix of functions involving the log-likelihood of two common distributions: a normal distribution and a categorical distribution with probabilities represented by a softmax function, which we refer to as a \emph{softmax distribution}. We show that the exact computations of the diagonals can be computed with linear complexity since computing the diagonal elements does not depend on off-diagonals in these cases. In the following, we consider the softmax and normal distributions, and we write the exact Hessian diagonals in both cases.

\subsection{Softmax distribution}
Consider a cross-entropy function for a discrete probability distribution as $f\doteq -\sum_{i=1}^{|\vq|}p_i\log q_i(\boldsymbol\theta)$, where $\vq$ is a probability vector that depends on a parameter vector $\boldsymbol\theta$, and $\vp$ is a one-hot vector for the target class. For softmax distributions, $\vq$ is parametrized by a softmax function $\vq \doteq {e^{\boldsymbol\theta}/\sum_{i=1}^{|{\boldsymbol q}|} e^{\theta_i}}$. In this case, we can write the gradient of the cross-entropy function with respect to $\boldsymbol\theta$ as
\[ \nabla_{\boldsymbol\theta}f(\boldsymbol\theta) = \vq-\vp.\]
Next, we write the exact diagonal elements of the Hessian matrix as follows:
\[ \operatorname{diag}(\mH_\vtheta) = \operatorname{diag}(\nabla_{\boldsymbol\theta}(\vq-\vp)) = \vq-\vq^2,\]
where $\vq^2$ denotes element-wise squaring of $\vq$, and $\nabla$ operator applied to a vector denotes Jacobian. Computing the exact diagonals of the Hessian matrix depends only on vector operations, which means that we can compute with linear complexity in the network's output dimension. The cross-entropy loss is used with softmax distribution in many important tasks, such as supervised classification and discrete reinforcement learning control with parameterized policies \cite{chan2022greedification}.
\subsection{Multivariate normal distribution with diagonal covariance}
For a multivariate normal distribution with diagonal covariance, the parameter vector $\boldsymbol\theta$ is determined by the mean-variance vector pair: $\boldsymbol\theta \doteq (\boldsymbol \mu, \vsigma^2)$. The log-likelihood of a random vector $\vx$ drawn from this distribution can be written as
\begin{align*}
\log q(\vx; \boldsymbol\mu, \vsigma^2) &= -\frac{1}{2} (\vx-\boldsymbol\mu)^\top \boldsymbol D(\vsigma^2)^{-1} (\vx-\boldsymbol\mu) - \frac{1}{2}\log (|\boldsymbol{D}(\vsigma^2)|) + c\\
&= -\frac{1}{2} (\vx-\boldsymbol\mu)^\top \boldsymbol D(\vsigma^2)^{-1} (\vx-\boldsymbol\mu) - \frac{1}{2}\log (\sum_{i=1}^{|\vsigma|} \sigma^2_i ) + c,
\end{align*}
where $\boldsymbol D(\vsigma^2)$ gives a diagonal matrix with $\vsigma^2$ in its diagonal, $|\mM|$ is the determinant of a matrix $\mM$ and $c$ is some constant. We can write the gradients of the log-likelihood function with respect to $\boldsymbol \mu$ and $\vsigma^2$ as follows:
\begin{align*}
\nabla_{\boldsymbol \mu}\log q(\vx; \boldsymbol\mu, \vsigma^2) &= \boldsymbol D(\vsigma^2)^{-1}  (\vx - \boldsymbol\mu) = (\vx - \boldsymbol\mu) \oslash \vsigma^2,\\
\nabla_{\vsigma^2}\log q(\vx; \boldsymbol\mu, \vsigma^2) &= \frac{1}{2} \big[(\vx - \boldsymbol\mu)^2 \oslash \vsigma^2 - \mathbf{1}\big] \oslash \vsigma^2,
\end{align*}
where $\mathbf{1}$ is an all-ones vector, and $\oslash$ denotes element-wise division.
Finally, we write the exact diagonals of the Hessian matrix as
\[ \operatorname{diag}(\mH_\vmu) = \operatorname{diag}(\nabla_{\vmu} (\vx - \boldsymbol\mu) \oslash \vsigma^2) = -\mathbf{1}\oslash\vsigma^2,\]
\[ \operatorname{diag}(\mH_{\vsigma^2}) = \operatorname{diag}\Big(\nabla_{\vsigma^2} \left[\frac{1}{2} \left[(\vx - \boldsymbol\mu)^2 \oslash \vsigma^2 - \mathbf{1}\right] \oslash \vsigma^2 \right]\Big) = \big[0.5\mathbf{1} -(\vx - \boldsymbol\mu)^2 \oslash \vsigma^2\big] \oslash \vsigma^4. \]
Clearly, the gradient and the exact Hessian diagonals can be computed with linear complexity in the network's output dimension. Log-likelihood functions for normal distributions are used in many important problems, such as variational inference and continuous reinforcement learning control. 

%%%%%%%%%%%%%%%%%%%%%%%%%%%%%%%%%%%

\section{HesScale for reinforcement learning loss functions}
\label{appendix:rl-loss-hesscale}
\subsection{Policy gradient loss}
The policy gradient loss is computed as the multiplication between the negative log-likelihood and some scalar value that determines how good or bad the action or trajectory selected is. The policy gradient loss is given by
\begin{align*}
    \mathcal{L}_\text{PG}(s,a) &= - \log q(a|s; \boldsymbol\mu, \vsigma^2) A
\end{align*}
where $A$ is the return, advantage, or TD error, depending on the algorithm used. The Hessian diagonals of such loss would differ by the multiplicative factor $A$ from the Hessian diagonals w.r.t.\ the log-likelihood defined in Appendix \ref{appendix:linear-likelihood-hessian} for both continuous or categorical distributions. The A2C \cite{mnih2016asynchronous} algorithm uses the advantage function for $A$.

\subsection{Value loss}
The value loss used in policy gradient methods can be implemented as a regular regression error between the value and its bootstrapped target. Thus, the Hessian diagonals of such loss function would be ones since the Hessian matrix is the identity.

\subsection{PPO policy loss}
The proximal policy optimization (PPO, \citealp{schulman2017proximal}) methods depend on a surrogate loss different from the one presented in the previous section, unlike A2C \cite{mnih2016asynchronous}. The surrogate loss depends on the ratio between the current action probability and the old action probability. Thus, we focus in this section on deriving the gradient and Hessian diagonals of a multivariate normal distribution with diagonal covariance. The Hessian diagonals of such loss would differ by the multiplicative factor $A$ from the Hessian diagonals w.r.t.\ the action probability. The action probability of a multivariate normal distribution with diagonal covariance is given by
\begin{align*}
p(\va;\boldsymbol\mu,\boldsymbol\sigma^2) &= \frac{1}{\sqrt{2\pi |\mD(\boldsymbol\sigma^2)|}} \exp^{\frac{1}{2} (\va-\boldsymbol\mu)^\top \boldsymbol D(\vsigma^2)^{-1} (\va-\boldsymbol\mu)}.
\end{align*}
We can write the gradients of the Gaussian probability function with respect to $\boldsymbol \mu$ and $\vsigma^2$ as follows:
\begin{align*}
    \nabla_{\boldsymbol\mu} p(\va;\boldsymbol\mu,\boldsymbol\sigma^2)
    &= p(\va;\boldsymbol\mu,\boldsymbol\sigma^2) (\va-\boldsymbol\mu) \oslash {\boldsymbol\sigma^2},
\end{align*}
\begin{align*}
    \nabla_{\boldsymbol\sigma^2} p(\va;\boldsymbol\mu,\boldsymbol\sigma^2)
    &= p(\va;\boldsymbol\mu,\boldsymbol\sigma^2) \left((\va-\boldsymbol\mu)^2 \oslash \boldsymbol\sigma^2 - 1 \right) \oslash (2\boldsymbol\sigma^2).
\end{align*}
Finally, we write the exact diagonals of the Hessian matrix as
\begin{align*}
    \operatorname{diag}(\mH_\vmu)  &= \left(\left(\va-\boldsymbol\mu\right)\circ \nabla_{\boldsymbol\mu} p(\va;\boldsymbol\mu,\boldsymbol\sigma^2)  - p(\va;\boldsymbol\mu,\boldsymbol\sigma^2) \right) \oslash \boldsymbol\sigma^2,
\end{align*}
\begin{align*}
    \operatorname{diag}(\mH_{\boldsymbol\sigma^2})
    &= (\boldsymbol\sigma^2 \circ \nabla_{\boldsymbol\sigma^2} p(\va;\boldsymbol\mu,\boldsymbol\sigma^2) -  p(\va;\boldsymbol\mu,\boldsymbol\sigma^2)) \left((\va-\boldsymbol\mu)^2 \oslash \boldsymbol\sigma^2 - 1\right) \oslash(2 \boldsymbol\sigma^4) - 0.5 p(\va;\boldsymbol\mu,\boldsymbol\sigma^2)(\va-\boldsymbol\mu)^2 \oslash \boldsymbol\sigma^6.
\end{align*}

%%%%%%%%%%%%%%%%%%%%%%%%%%%%%%%%%%%

\section{Scalability}
\label{appendix:scalability}
We perform another experiment to evaluate the computational cost of our optimizers.
Our Hessian approximation methods and corresponding optimizers have linear computational complexity, which can be seen from Eq.\ \ref{equ:hessian-approx-a} and Eq.\ \ref{equ:hessian-approx-W}.
However, computing second-order information in optimizers still incurs extra computations compared to first-order optimizers, which may impact how the total computations scale with the number of parameters.
Hence, we compare the computational cost of our optimizers with others for various numbers of parameters. More specifically, we measure the update time of each optimizer, which is the time needed to backpropagate first-order and second-order information and update the parameters.

We designed two experiments to study the computational cost of first-order and second-order optimizers. In the first experiment, we used a neural network with a single hidden layer. The network has 64 inputs and 512 hidden units with \emph{tanh} activations. We study the increase in computational time when increasing the number of outputs exponentially, which roughly doubles the number of parameters. The set of values we use for the number of outputs is $\{2^4, 2^5, 2^6, 2^7, 2^8, 2^9\}$. The results of this experiment are shown in Fig.\ \ref{fig:comp-cost-single-layer}. In the second experiment, we used a neural network with multi-layers, each containing 512 hidden units with \emph{tanh} activations. The network has 64 inputs and 100 outputs. We study the increase in computational time when increasing the number of layers exponentially, which also roughly doubles the number of parameters. The set of values we use for the number of layers is $\{1, 2, 4, 8, 16, 32, 64, 128\}$. The results are shown in Fig.\ \ref{fig:comp-cost-layers}. The points in Fig.\ \ref{fig:comp-cost-single-layer} and Fig.\ \ref{fig:comp-cost-layers} are averaged over 30 updates. The standard errors of the means of these points are smaller than the width of each line. On average, we notice that the cost of AdaHessian, AdaHesScale, and AdaHesScaleGN are three, two, and 1.25 times the cost of Adam, respectively. It is clear that our methods are among the most computationally efficient approximation method for Hessian diagonals.
\begin{figure}[ht]
\centering
\subfigure[Increasing number of outputs in a neural network]{
     \includegraphics[width=0.45\textwidth]{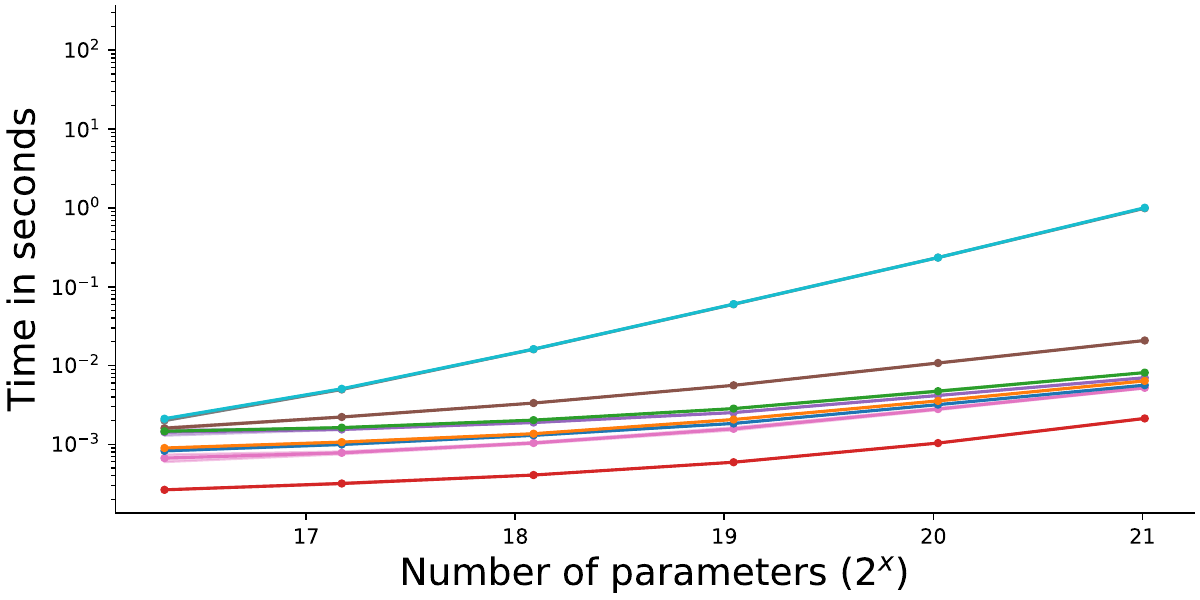}
     \label{fig:comp-cost-single-layer}
}
\subfigure[Increasing number of layers in a neural network]{
     \centering
     \includegraphics[width=0.45\textwidth]{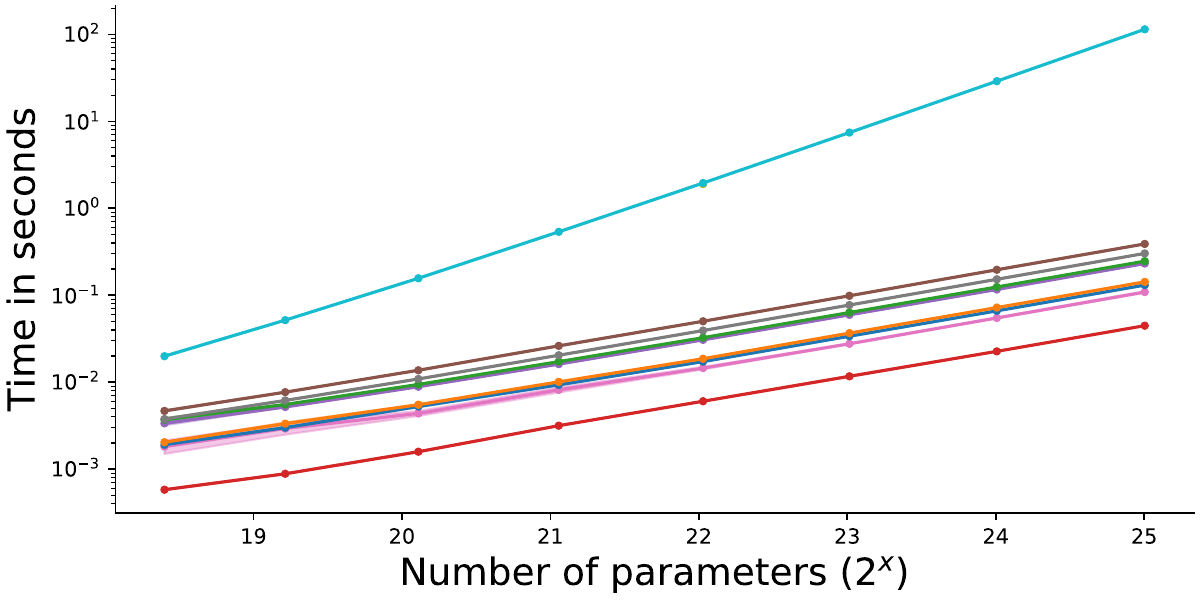}
     \label{fig:comp-cost-layers}
}
 \includegraphics[width=0.4\textwidth]{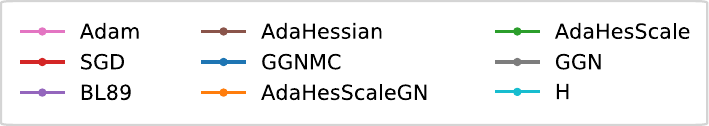}
\caption{The average computation time for each step of an update is shown for different optimizers. The computed update time is the time needed by each optimizer to backpropagate gradients or second-order information and to update the parameters of the network. GGN overlaps with H in (a).}
\label{fig:computational_cost}
\end{figure}
\newpage

\section{Conditions of Weight Structure for Hessian Diagonality}
\label{appendix:diagonality-conditions}
This section aims to provide a theoretical analysis of the weight structure that makes the Hessian blocks with respect to the pre-activations diagonal or dominantly diagonal. Assuming that the matrix $\mH_{l}\in\mathbb{R}^{n\times n}$ containing second derivatives at the $l$-layer is diagonal, we look for the conditions on $\mW_{l}\in\mathbb{R}^{n\times k}$ and $\mH_{l}$ that guarantee $\mH_{l-1}\in \mathbb{R}^{k\times k}$ to be diagonal. Formally, we want to analyze: $\min_{\mW_l} \|(\mW_{l}^\top \mH_{l} \mW_{l})\circ \mI - \mW_{l}^\top \mH_{l} \mW_{l} \|^2_F$, where $\circ$ denotes the Hadamard product.

Let us start by writing the quantity $J$ we want to minimize and simplify the expression. We let $\mA=\mW_{l}^\top \mH_{l} \mW_{l}$ and write the expression as follows:
\begin{align*}
    J &= \|(\mW_{l}^\top \mH_{l} \mW_{l})\circ \mI - \mW_{l}^\top \mH_{l} \mW_{l} \|^2_F\\
    &= \|\mA_l\circ \mI - \mA_l \|^2_F \shorteqnote{(substitute $\mW_{l}^\top \mH_{l} \mW_{l}$ by $\mA_l$)} \\
    &= Tr\left(\left(\mA_l\circ \mI - \mA_l )^\top (\mA_l\circ \mI - \mA_l \right)\right)  \shorteqnote{(remember that $\|\mA_l\|^2_F=Tr(\mA_l^\top \mA_l)$)} \\
    &= Tr\left(\left(\mA_l^\top\circ \mI - \mA_l^\top ) (\mA_l\circ \mI - \mA_l \right)\right) \\
    &= Tr\left(\left(\mA_l^2 \circ \mI -\mA_l^2 \circ \mI -\mA_l^2 \circ \mI +\mA_l^\top \mA_l \right)\right) \shorteqnote{($\mA_l^2$ denotes element-wise squaring)} \\
    &= Tr\left(\left(\mA_l^\top \mA_l -\mA_l^2 \circ \mI \right)\right). \shorteqnote{(expectation on the sum of squared off-diagonal elements)}
\end{align*}
Now we can write $\nabla_{\mW_l} J$ as follows:
\begin{align*}
    \nabla_{\mW_l} J &= Tr\left(\mA_l^\top \mA_l -\mA_l^2 \circ \mI \right) \\
    &= \nabla_{\mW_l}Tr\left(\mA_l^\top \mA_l -\mA_l^2 \circ \mI \right)\\
    &= \nabla_{\mW_l}Tr\left(\mA_l^\top \mA_l\right) - \nabla_{\mW_l} Tr\left(\mA_l^2 \circ \mI \right). \shorteqnote{$\left(Tr(\mA +\mB) = Tr(\mA) + Tr(\mB)\right)$}
\end{align*}
Taking the derivative with respect to $\mW_l$ involves the $4$-tensor of derivatives of elements of $\mA_l$ with respect to $\mW_l$, so using the index notation for the following calculations is more convenient. We drop the subscript $l$ for clarity but emphasize that all matrices have a subscript $l$. We first write the trace of the first and second term in the index notation as follows:
\begin{align*}
    A_{i,j} = \sum_{l=1}^k W_{l,i}W_{l,j}H_{l,l}, \quad\quad \Gamma = Tr(\mA^\top \mA) = \sum_{i=1}^{k} \sum_{j=1}^{k} A^2_{i,j}, \quad \quad \Lambda = Tr(\mA^2\circ \mI) = \sum_{i=1}^{k} A^2_{i,i}.
\end{align*}

We would like to compute $\nabla_{\mW} \Gamma$ and $\nabla_{\mW} \Lambda$. One can use the chain rule as follows:
\begin{align*}
    \frac{\partial \square}{\partial W_{i,j}} &= \sum_{m=1}^k \sum_{n=1}^k \frac{\partial \square}{\partial A_{m,n}} \frac{\partial A_{m,n}}{\partial W_{i,j}}
\end{align*}
where $\square\in\{\Gamma, \Lambda\}$. Let us now calculate $\nabla_{\mA} \Gamma$ and $\nabla_{\mA} \Lambda$ as follows:
\begin{align*}
    \frac{\partial \Gamma}{\partial A_{m,n}} &= \frac{\partial}{\partial A_{m,n}} \sum_{i=1}^{k} \sum_{j=1}^{k} A^2_{i,j} = 2 A_{m,n},\\
    \frac{\partial \Lambda}{\partial A_{m,n}} &= \frac{\partial}{\partial A_{m,n}} \sum_{i=1}^{k} A^2_{i,i} = 2 A_{m,n} \delta_{m,n},
\end{align*}
where $\delta_{m,n}=1$ when $m=n$ and $0$ otherwise. Let us now calculate the elements of the second term in the chain rule, which is a $4$-tensor, as follows:
\begin{align*}
    \frac{\partial A_{m,n}}{\partial W_{i,j}} &= \frac{\partial}{\partial W_{i,j}} \left(\sum_{l=1}^k W_{l,m}W_{l,n}H_{l,l}\right)\\
    &= \sum_{l=1}^k \frac{\partial W_{l,m}}{\partial W_{i,j}} W_{l,n} H_{l,l} + \sum_{l=1}^k \frac{\partial W_{l,n}}{\partial W_{i,j}} W_{l,m} H_{l,l}  + \sum_{l=1}^k \frac{\partial H_{l,l}}{\partial W_{i,j}} W_{l,m} W_{l,n}\\
    &= \sum_{l=1}^k W_{l,n} H_{l,l} \delta_{j,m} \delta_{i,l} + \sum_{l=1}^k W_{l,m} H_{l,l} \delta_{i,l} \delta_{j,n} + R_{i,j}^{m,n} \\
    &= W_{i,n} H_{i,i} \delta_{j,m} + W_{i,m} H_{i,i} \delta_{j,n} + R_{i,j}^{m,n},
\end{align*}
where $R_{i,j}^{m,n}= \sum_{l=1}^k \frac{\partial H_{l,l}}{\partial W_{i,j}} W_{m,l} W_{n,l}$ denotes the $4$-tensor containing the derivative of the Hessian diagonals of the pre-activations with respect to the weights.

Now, we are able to compute $\nabla_{\mW} \Gamma$ and $\nabla_{\mW} \Lambda$ as follows:
\begin{align*}
    \frac{\partial \Gamma}{\partial W_{i,j}} &= \sum_{m=1}^k \sum_{n=1}^k \frac{\partial \Gamma}{\partial A_{m,n}} \frac{\partial A_{m,n}}{\partial W_{i,j}}\\
    &= \sum_{m=1}^k \sum_{n=1}^k 2 A_{m,n} \left(W_{i,n} H_{i,i} \delta_{j,m} + W_{i,m} H_{i,i} \delta_{j,n} + R_{i,j}^{m,n}\right)\\
    &= 2\sum_{n=1}^k A_{j,n} W_{i,n} H_{i,i} + 2\sum_{m=1}^k A_{m,j} W_{i,m} H_{i,i} + 2\sum_{m=1}^k \sum_{n=1}^k A_{m,n}  R_{i,j}^{m,n} \\
    &= 4\sum_{n=1}^k A_{j,n} W_{i,n} H_{i,i} + 2\sum_{m=1}^k \sum_{n=1}^k A_{m,n} R_{i,j}^{m,n} \shorteqnote{(note that $A_{m,i}=A_{i,m}$ since $\mA$ is symmetric)}
\end{align*}
\begin{align*}
    \frac{\partial \Lambda}{\partial W_{i,j}} &= \sum_{m=1}^k \sum_{n=1}^k \frac{\partial \Lambda}{\partial A_{m,n}} \frac{\partial A_{m,n}}{\partial W_{i,j}}\\
    &= \sum_{m=1}^k \sum_{n=1}^k 2 A_{m,n} \delta_{m,n} \left(W_{i,n} H_{i,i} \delta_{j,m} + W_{i,m} H_{i,i} \delta_{j,n} + R_{i,j}^{m,n}\right)\\
    &= 2\sum_{m=1}^k A_{m,m} \left(W_{i,m} H_{i,i} \delta_{j,m} + W_{i,m} H_{i,i} \delta_{j,m} + R_{i,j}^{m,m}\right)\\
    &= 2\sum_{m=1}^k A_{m,m} \left(2W_{i,m} H_{i,i} \delta_{j,m} + R_{i,j}^{m,m}\right) \\
    &= 4 A_{j,j} W_{i,j} H_{i,i} + 2\sum_{m=1}^{k} A_{m,m} R_{i,j}^{m,m},
\end{align*}
\begin{align*}
    \frac{\partial (\Gamma-\Lambda)}{\partial W_{i,j}}
    &= 4\sum_{n=1}^k A_{j,n} W_{i,n} H_{i,i} + 2\sum_{m=1}^k \sum_{n=1}^k A_{m,n} R_{i,j}^{m,n} - 4 A_{j,j} W_{i,j} H_{i,i} - 2\sum_{m=1}^{k} A_{m,m} R_{i,j}^{m,m} \\
    &= 4\sum_{n=1}^k A_{j,n} W_{i,n} H_{i,i} - 4 A_{j,j} W_{i,j} H_{i,i} +  2\sum_{m=1}^k \sum_{n=1}^k A_{m,n} R_{i,j}^{m,n} - 2\sum_{m=1}^{k} A_{m,m} R_{i,j}^{m,m} \\
    &= 4\sum_{n=1}^k A_{j,n} W_{i,n} H_{i,i} - 4 \sum_{n=1}^k A_{j,j} W_{i,j} H_{i,i} \delta_{j,n} +  2\sum_{m=1}^k \sum_{n=1}^k A_{m,n} R_{i,j}^{m,n} - 2\sum_{m=1}^{k} \sum_{n=1}^k A_{m,n} R_{i,j}^{m,n} \delta_{m,n} \\
    &= 4\sum_{n=1}^k A_{j,n} W_{i,n} H_{i,i} (1-\delta_{j,n}) +  2\sum_{m=1}^k \sum_{n=1}^k A_{m,n} R_{i,j}^{m,n}(1-\delta_{m,n}).
\end{align*}
We can finally write the element of the gradient $\nabla_{\mW} J$ as follows:
\begin{align*}
    [\nabla_{\mW} J]_{i,j} &= \nabla_{\mW}(\Gamma - \Lambda)\\
    &=4\sum_{n=1}^k A_{j,n} W_{i,n} H_{i,i} (1-\delta_{j,n}) +  2\sum_{m=1}^k \sum_{n=1}^k A_{m,n} R_{i,j}^{m,n}(1-\delta_{m,n}).
\end{align*}
We find the conditions on $\mW$ and $\mH$ that minimizes $J$ by setting $\nabla_{\mW} J$ to zero. There are two notable cases where the gradient equals zero.

\subsection{When $\mW$ is diagonal}
If $\mW$ is diagonal, we can write it as $W_{i,j}=\delta_{i,j}$. Therefore, the gradient elements are reduced to:
\begin{align*}
    [\nabla_{\mW} J]_{i,j} &= 4\sum_{n=1}^k \left(A_{j,n} W_{i,n} H_{i,i}\right)\left(1-\delta_{j,n}\right) + 2\sum_{m=1}^k \sum_{n=1}^k A_{m,n} R_{i,j}^{m,n}\left(1 - \delta_{m,n} \right)\\
    &= 4\sum_{n=1}^k A_{j,n} H_{i,i} \delta_{i,n} (1-\delta_{n,j}) + 2\sum_{m=1}^k \sum_{n=1}^k A_{m,n} R_{i,j}^{m,n}\left(1 - \delta_{m,n} \right) \\
    &= 4\sum_{n=1}^k \left(\sum_{l=1}^k W_{l,j}W_{l,n}H_{l,l}\right) H_{i,i} \delta_{i,n} (1-\delta_{n,j}) + 2\sum_{m=1}^k \sum_{n=1}^k \left(\sum_{l=1}^k W_{l,m}W_{l,n}H_{l,l}\right) R_{i,j}^{m,n}\left(1 - \delta_{m,n} \right)\\
    &= 4\sum_{n=1}^k \left(\sum_{l=1}^k \delta_{l,j}\delta_{l,n}H_{l,l}\right) H_{i,i} \delta_{i,n} (1-\delta_{n,j}) + 2\sum_{m=1}^k \sum_{n=1}^k \left(\sum_{l=1}^k \delta_{l,m}\delta_{l,n}H_{l,l}\right) R_{i,j}^{m,n}\left(1 - \delta_{m,n} \right)\\
    &= 4\sum_{n=1}^k \left(\delta_{j,n}H_{j,j}\right) H_{i,i} \delta_{i,n} (1-\delta_{n,j}) + 2\sum_{m=1}^k \sum_{n=1}^k \left( \delta_{m,n}H_{m,m}\right) R_{i,j}^{m,n}\left(1 - \delta_{m,n} \right)\\
    &= 4H_{j,j} H_{i,i} \delta_{i,j} (1-\delta_{i,j}) + 2\sum_{m=1}^k \sum_{n=1}^k H_{m,m} R_{i,j}^{m,n}\delta_{m,n}\left(1 - \delta_{m,n} \right)\\
    &=0 \shorteqnote{(note that $\delta_{i,j} (1-\delta_{i,j}) = 0, \forall i,j$)}.
\end{align*}

\subsection{When $\mH=\alpha \mI$ and $W_{i,j}\sim \mathcal{N}(0,\sigma), \forall \sigma $, the gradient becomes zero on expectation}
We can write $\mH =\alpha \mI$ as $H_{i,j} = \alpha\delta_{i,j}$.  Therefore, the gradient elements are reduced to:
\begin{align*}
    \mathbb{E}_{\mW_{l}}[\nabla_{\mW} J]_{i,j} &= \mathbb{E}_{\mW_{l}}\left[4\sum_{n=1}^k A_{j,n} W_{i,n} H_{i,i} (1-\delta_{j,n}) +  2\sum_{m=1}^k \sum_{n=1}^k A_{m,n} R_{i,j}^{m,n}(1-\delta_{m,n})\right]\\
    &= \mathbb{E}_{\mW_{l}}\left[4\alpha\sum_{n=1}^k A_{j,n} W_{i,n} \delta_{i,i} (1-\delta_{j,n}) +  2\sum_{m=1}^k \sum_{n=1}^k A_{m,n} R_{i,j}^{m,n}(1-\delta_{m,n}) \right] \\
    &= \mathbb{E}_{\mW_{l}}\Bigg[4\alpha^2\sum_{n=1}^k \Bigg(\sum_{l=1}^k W_{l,j}W_{l,n}\delta_{l,l}\Bigg) W_{i,n}  \delta_{i,i} (1-\delta_{j,n}) \\& \quad\quad\quad\quad\quad + 2\alpha\sum_{m=1}^k \sum_{n=1}^k \left(\sum_{l=1}^k W_{l,m}W_{l,n}\delta_{l,l}\right) R_{i,j}^{m,n}\left(1 - \delta_{m,n} \right)\Bigg]\\
    &= \mathbb{E}_{\mW_{l}}\left[4\alpha^2\sum_{n=1}^k \left(\sum_{l=1}^k W_{l,j}W_{l,n}\right) W_{i,n}(1-\delta_{j,n}) + 2\alpha\sum_{m=1}^k \sum_{n=1}^k \left(\sum_{l=1}^k W_{l,m}W_{l,n}\right) R_{i,j}^{m,n}\left(1 - \delta_{m,n} \right)\right]\\
    &= \mathbb{E}_{\mW_{l}}\left[4\alpha^2 W_{i,j}^3 - 4\alpha^2 W_{i,j}^3 + 2\alpha\sum_{m=1}^k \left(\sum_{l=1}^k W_{l,m}W_{l,m}\right) R_{i,j}^{m,m} - 2\alpha\sum_{n=1}^k \left(\sum_{l=1}^k W_{l,m}W_{l,m}\right) R_{i,j}^{m,m} \right]\\
    &= 0.
\end{align*}
Note that $\E[{W_{i,j}^2 W_{i,l}}]=0, \forall l\neq j$.

%%%%%%%%%%%%%%%%%%%%%%%%%%%%%%%%%%%

\section{HesScale with Convolutional Neural Networks}
\label{appendix:cnn-derivation}
Here, we derive the Hessian propagation for convolutional neural networks (CNNs). Consider a CNN with $L-1$ layers followed by a fully connected layer that outputs the predicted output $\vq$. The CNN filters are parameterized by $\{\mW_1,...,\mW_L\}$, where $\mW_l$ is the filter matrix at the $l$-th layer with the dimensions $k_{l,1}\times k_{l,2}$, and its element at the $i$th row and the $j$th column is denoted by $W_{l,i,j}$. For the simplicity of this proof, we assume that the number of filters at each layer is one; the proof can be extended easily to the general case. At the layer $l$, we get the activation matrix $\mH_{l}$ by applying the activation function $\boldsymbol\sigma$ to the pre-activation $\mA_{l}$: $\mH_{l} = \boldsymbol\sigma(\mA_{l})$. We assume here that the activation function is element-wise activation for all layers except for the final layer $L$, where it becomes the softmax function. We simplify notations by defining $\mH_0 \doteq \mX$, where $\mX$ is the input sample. The activation $\mH_{l}$ is then convoluted by the weight matrix $\mW_{l+1}$ of layer $l+1$ to produce the next pre-activation: ${A}_{l+1,i,j} = \sum_{m=0}^{k_{l,1} - 1}\sum_{n=0}^{k_{l,2} - 1} {{W}_{l+1,m,n}}{H}_{l,(i+m),(j+n)}$. We denote the size of the activation at the $l$-th layer by $h_{l}\times w_{l}$. The recursive formulation used for Hessian diagonals backpropagation is given in \cref{thm:hesscale-cnn}.

\begin{theorem}
\label{thm:hesscale-cnn}
\textbf{HesScale Computation with CNNs}. Under the zero second-order off-diagonals assumption in all layers of a neural network except for the last one, the second derivatives can be propagated with linear complexity in the number of network parameters and in the network's output dimension using the following equations:

\begin{align*}
\widehat{\frac{\partial^2 \mathcal{L}}{\partial {A^2_{l,i,j}}}} &\doteq  \sigma^\prime(A_{l, i,j})^2 \sum_{m=0}^{k_{l+1,2} - 1} \sum_{n=0}^{k_{l+1,2} - 1} \widehat{\frac{\partial^2 \mathcal{L}}{\partial A^2_{l+1, (i-m),(j-n)}}} W^2_{l+1, m,n} \nonumber \\
& +  \sigma^{\prime\prime}(A_{l, i,j})\sum_{m=0}^{k_{l+1,2} - 1} \sum_{n=0}^{k_{l+1,2} - 1} \frac{\partial \mathcal{L}}{\partial A_{l+1, (i-m),(j-n)}} W_{l+1, m,n},\\
\widehat{\frac{\partial^2 \mathcal{L}}{\partial {W^2_{l,i, j}}}} &\doteq \sum_{m=0}^{h_{l} - k_{l,1}}\sum_{n=0}^{w_{l} - k_{l,2}} \widehat{\frac{\partial^2 \mathcal{L}}{\partial A^2_{l, m,n}}} H^2_{l-1, (i+m),(j+n)}.
\end{align*}
\end{theorem}
\begin{proof}
The backpropagation equations for the described network are given following \citet{rumelhart1986learning}:
\begin{align}
\frac{\partial \mathcal{L}}{\partial A_{l,i,j}} &=  \sum_{m=0}^{k_{l+1,1} - 1}\sum_{n=0}^{k_{l+1,2} - 1} \frac{\partial \mathcal{L}}{\partial A_{l+1, (i-m),(j-n)}} \frac{\partial A_{l+1, (i-m),(j-n)}}{\partial A_{l, i,j}} \nonumber\\
&= \sum_{m=0}^{k_{l+1,1} - 1}\sum_{n=0}^{k_{l+1,2} - 1} \frac{\partial \mathcal{L}}{\partial A_{l+1, (i-m),(j-n)}} \sum_{m'=0}^{k_{l+1,1} - 1}\sum_{n'=0}^{k_{l+1,2} - 1} W_{l+1, m',n'} \frac{\partial H_{l, (i-m+m'), (j-n+n')}}{\partial A_{l, i, j}}  \nonumber \\
&= \sum_{m=0}^{k_{l+1,1} - 1}\sum_{n=0}^{k_{l+1,2} - 1} \frac{\partial \mathcal{L}}{\partial A_{l+1, (i-m),(j-n)}} W_{l+1, m,n} \sigma^\prime(A_{l, i,j}) \nonumber \\
&= \sigma^\prime(A_{l, i,j}) \sum_{m=0}^{k_{l+1,1} - 1}\sum_{n=0}^{k_{l+1,2} - 1} \frac{\partial \mathcal{L}}{\partial A_{l+1, (i-m),(j-n)}} W_{l+1, m,n},
\end{align}
\begin{align}
\frac{\partial \mathcal{L}}{\partial W_{l,i,j}} &= \sum_{m=0}^{h_{l} - k_{l,1}}\sum_{n=0}^{w_{l} - k_{l,2}} \frac{\partial \mathcal{L}}{\partial A_{l, m,n}} \frac{\partial A_{l, m,n}}{\partial W_{l, i,j}} =  \sum_{m=0}^{h_{l} - k_{l,1}}\sum_{n=0}^{w_{l} - k_{l,2}} \frac{\partial \mathcal{L}}{\partial A_{l, m,n}} H_{l-1, (i+m),(j+n)}.
\end{align}
In the following, we write the equations for the exact Hessian diagonals with respect to weights  $\sfrac{\partial^2 \mathcal{L}}{\partial {W^2_{l,i, j}}}$, which requires the calculation of $\sfrac{\partial^2 \mathcal{L}}{\partial {A^2_{l,i,j}}}$ first:

\begin{align}
\frac{\partial^2 \mathcal{L}}{\partial {A^2_{l,i,j}}}
&= \frac{\partial}{\partial A_{l,i,j}} \Bigg[ \sigma^\prime(A_{l, i,j}) \sum_{m=0}^{k_{l+1,1} - 1}\sum_{n=0}^{k_{l+1,2} - 1} \frac{\partial \mathcal{L}}{\partial A_{l+1, (i-m),(j-n)}} W_{l+1, m,n}   \Bigg] \nonumber \\
&= \sigma^\prime(A_{l, i,j}) \sum_{m,p=0}^{k_{l+1,2} - 1} \sum_{n,q=0}^{k_{l+1,2} - 1} \frac{\partial^2 \mathcal{L}}{\partial A_{l+1, (i-m),(j-n)}\partial A_{l+1, (i-p),(j-q)}} \frac{\partial A_{l+1, (i-p),(j-q)}}{\partial A_{l,i,j}}  W_{l+1, m,n}   \nonumber \\
&\quad +\sigma^{\prime\prime}(A_{l, i,j}) \sum_{m=0}^{k_{l+1,2} - 1} \sum_{n=0}^{k_{l+1,2} - 1} \frac{\partial \mathcal{L}}{\partial A_{l+1, (i-m),(j-n)}} W_{l+1, m,n}  \nonumber
\end{align}
\begin{align}
\frac{\partial^2 \mathcal{L}}{\partial {W^2_{l,i,j}}}
&= \frac{\partial}{\partial W_{l,i,j}} \Bigg[\sum_{m=0}^{h_{l} - k_{l,1}}\sum_{n=0}^{w_{l} - k_{l,2}} \frac{\partial \mathcal{L}}{\partial A_{l, m,n}} H_{l-1, (i+m),(j+n)} \Bigg] \nonumber \\
&= \sum_{m,p=0}^{h_{l} - k_{l,1}}\sum_{n,q=0}^{w_{l} - k_{l,2}} \frac{\partial^2 \mathcal{L}}{\partial A_{l, m,n}\partial A_{l, p,q}}\frac{\partial A_{l, p,q}}{\partial W_{l, i,j}} H_{l-1, (i+m),(j+n)} \nonumber
\end{align}

Since the calculation of $\sfrac{\partial^2 \mathcal{L}}{\partial {A^2_{l,i,j}}}$ and $\sfrac{\partial^2 \mathcal{L}}{\partial {W^2_{l,i,j}}}$ depend on the off-diagonal terms, the computation complexity becomes quadratic. 
Following \citet{becker1988improving}, we approximate the Hessian diagonals by ignoring the off-diagonal terms, which leads to a backpropagation rule with linear computational complexity for our estimates $\widehat{\frac{\partial^2 \mathcal{L}}{\partial {W^2_{l,i, j}}}}$ and $\widehat{\frac{\partial^2 \mathcal{L}}{\partial {A^2_{l,i,j}}}}$:
\begin{align}
\widehat{\frac{\partial^2 \mathcal{L}}{\partial {A^2_{l,i,j}}}} &\doteq  \sigma^\prime(A_{l, i,j})^2 \sum_{m=0}^{k_{l+1,2} - 1} \sum_{n=0}^{k_{l+1,2} - 1} \widehat{\frac{\partial^2 \mathcal{L}}{\partial A^2_{l+1, (i-m),(j-n)}}} W^2_{l+1, m,n} \nonumber \\
&\quad +\sigma^{\prime\prime}(A_{l, i,j})\sum_{m=0}^{k_{l+1,2} - 1} \sum_{n=0}^{k_{l+1,2} - 1} \frac{\partial \mathcal{L}}{\partial A_{l+1, (i-m),(j-n)}} W_{l+1, m,n} \label{equ:hessian-approx-a-conv},\\
\widehat{\frac{\partial^2 \mathcal{L}}{\partial {W^2_{l,i, j}}}} &\doteq \sum_{m=0}^{h_{l} - k_{l,1}}\sum_{n=0}^{w_{l} - k_{l,2}} \widehat{\frac{\partial^2 \mathcal{L}}{\partial A^2_{l, m,n}}} H^2_{l-1, (i+m),(j+n)}. \label{equ:hessian-approx-W-conv}
\end{align}
\end{proof}

%%%%%%%%%%%%%%%%%%%%%%%%%%%%%%%%%%%

\clearpage

\section{Additional Supervised Classification Results}
\label{appendix:additional-classification}
Here, we provide the sensitivity of the step size for each method in Fig.\ \ref{fig:sensitivity_epochs_cifar100}.

\begin{figure}[ht]
\centering
\subfigure[CIFAR-100 All-CNN]{
     \includegraphics[width=0.3\textwidth]{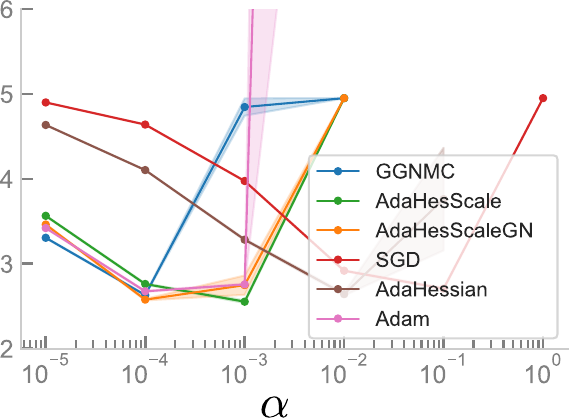}
     \label{fig:cifar100-all-cnn-epochs-sens}
}
\subfigure[CIFAR-100 3C3D]{
     \includegraphics[width=0.3\textwidth]{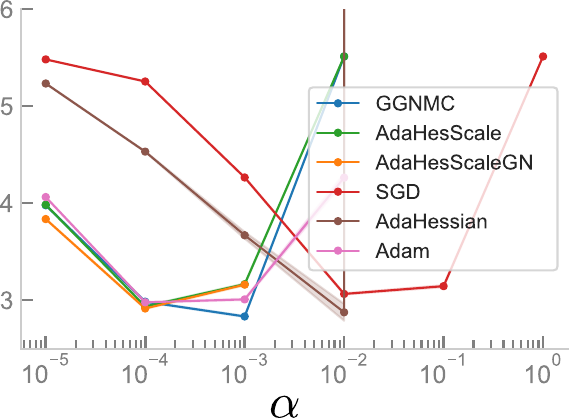}
     \label{fig:CIFAR-100-3c3d-epochs-sens}
}
\caption{Parameter Sensitivity study for each algorithm on CIFAR-100 with All-CNN and 3C3D architectures. The range of step size is $\{10^{-5}, 10^{-4}, 10^{-3}, 10^{-2}, 10^{-1} , 10^{0} \}$. We choose $\beta_1$ to be equal to $0.9$ and $\beta_2$ to be equal to $0.999$. Each point for each algorithm represents the average test loss given a set of parameters.}
\label{fig:sensitivity_epochs_cifar100}
\end{figure}

%%%%%%%%%%%%%%%%%%%%%%%%%%%%%%%%%%%

\section{Additional RL experiments}
\label{appendix:rl-additional}
Here, we provide additional results on the performance of the A2C and PPO algorithms in Fig.\ \ref{fig:rl-ppo-a2c-2}. In addition, we provide a robustness analysis of the step size with the A2C and PPO algorithms using Scaled AdaHesScale against Scaled Adam in Fig.\ \ref{fig:rl-robustness-2}, \ref{fig:rl-robustness-3}, and \ref{fig:rl-robustness-4}.

\begin{figure}[ht]
\centering
\includegraphics[width=\textwidth]{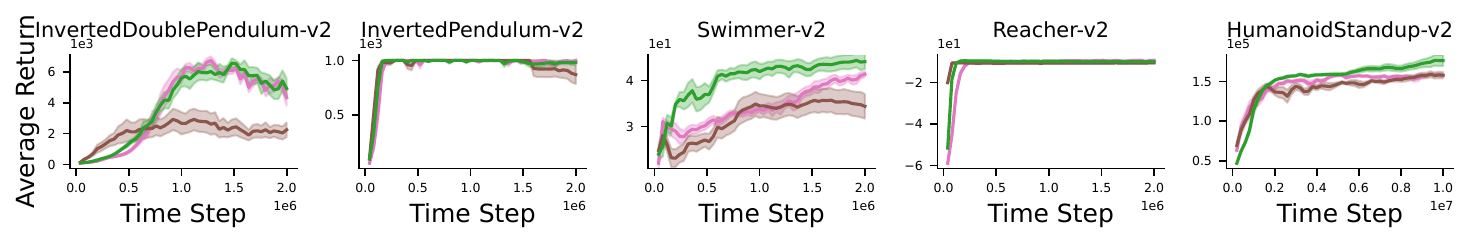}
\includegraphics[width=\textwidth]{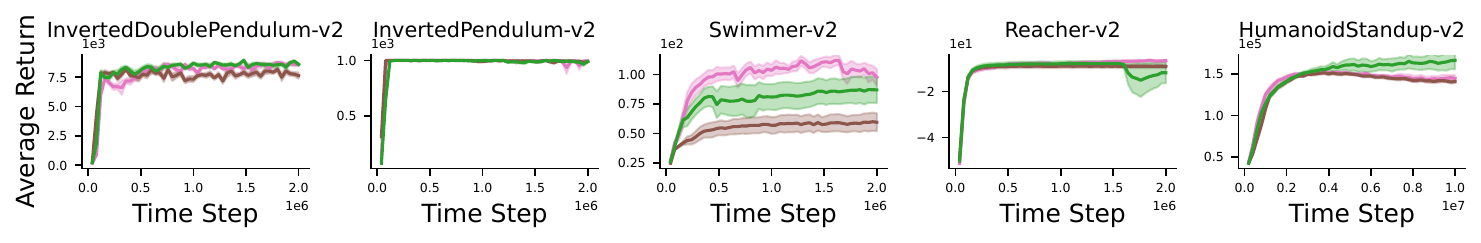}
\includegraphics[width=0.4\textwidth]{figures/rl/legends-rl.pdf}
\caption{Performance of A2C (first row) and PPO (second row) with AdaHesScale, Adam, and with AdaHessian on $5$ MuJoCo environments. We show the undiscounted return averaged over $10$ independent runs. The shaded area represents the standard error.}
\label{fig:rl-ppo-a2c-2}
\end{figure}

\begin{figure}[ht]
\centering
\includegraphics[width=0.65\textwidth]{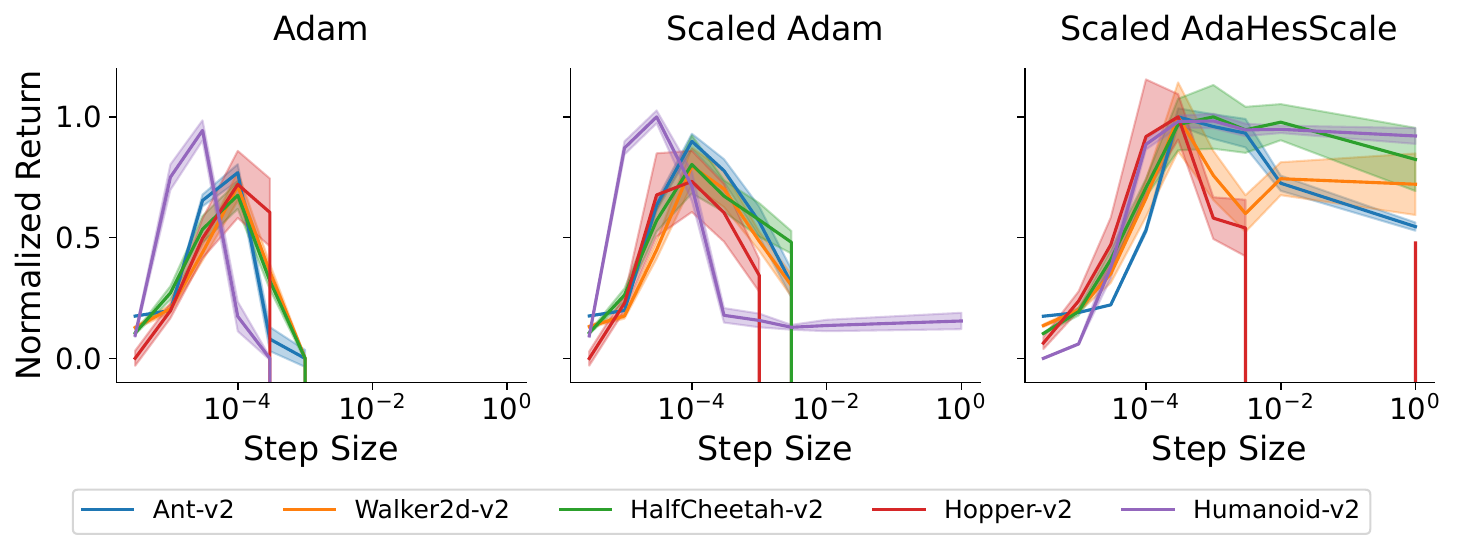}
\vspace{-0.5cm}
\caption{Robustness of HesScale-based step-size scaling with AdaHesScale and Adam on $5$ MuJoCo environments using A2C. We show the undiscounted return averaged over $10$ independent runs. The shaded area represents the standard error.}
\label{fig:rl-robustness-2}
\end{figure}

\begin{figure}[ht]
\centering
\includegraphics[width=0.8\textwidth]{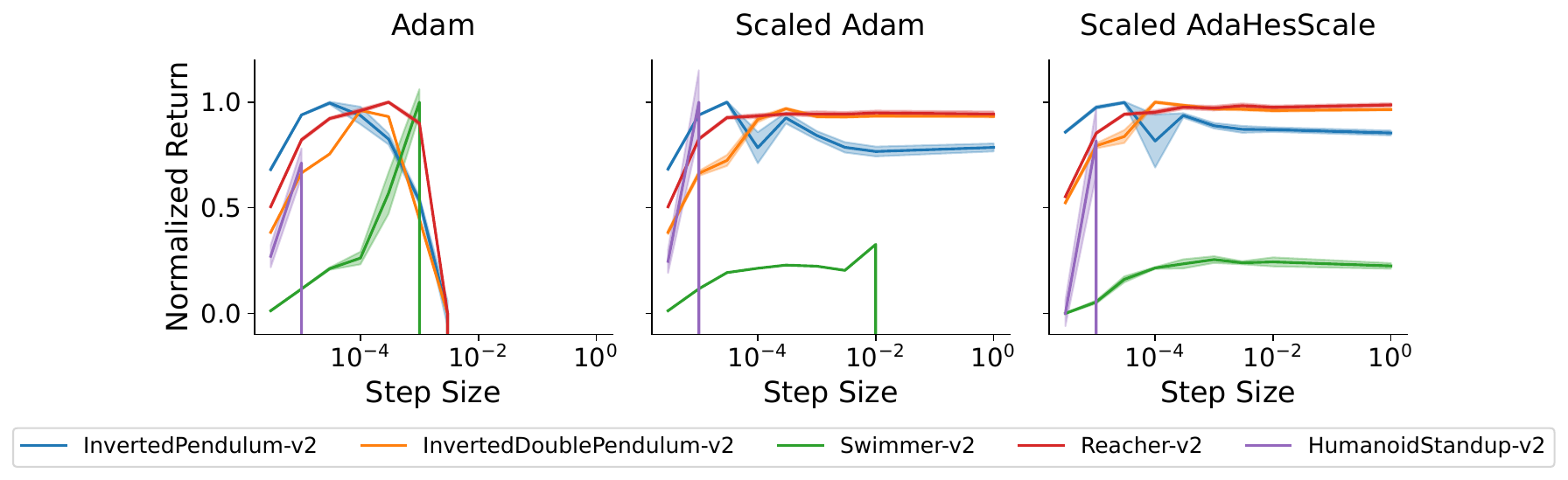}
\vspace{-0.5cm}
\caption{Robustness of HesScale-based step-size scaling with AdaHesScale and Adam on $5$ additional MuJoCo environments using PPO. We show the undiscounted return averaged over $10$ independent runs. The shaded area represents the standard error.}
\label{fig:rl-robustness-3}
\end{figure}

\begin{figure}[ht]
\centering
\includegraphics[width=0.8\textwidth]{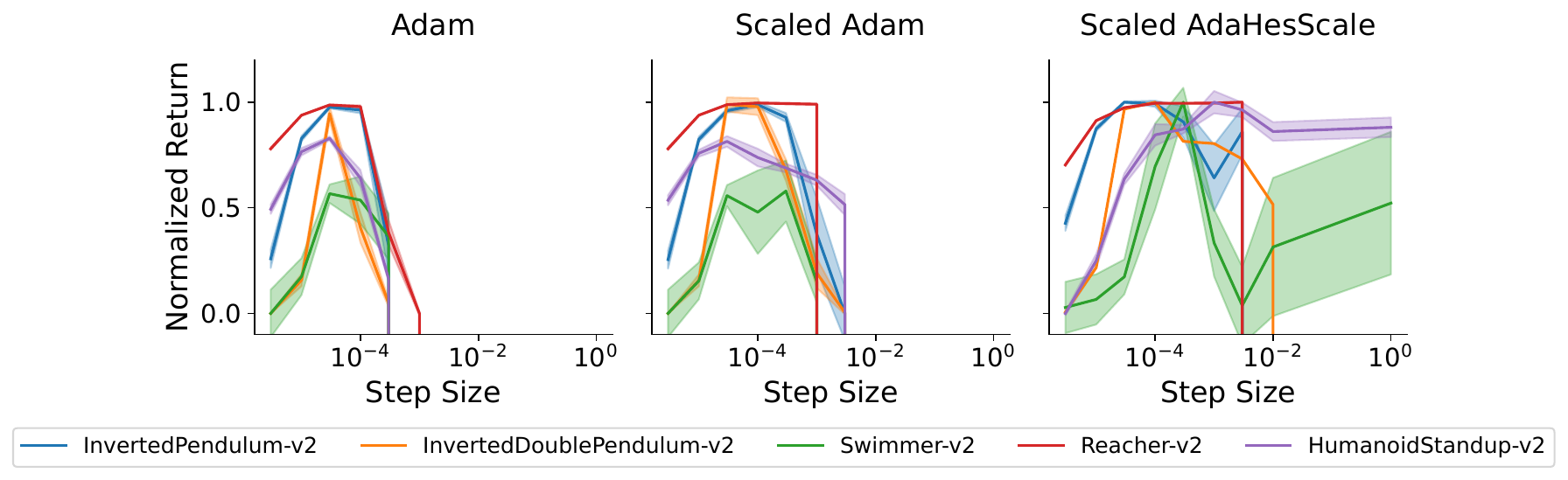}
\vspace{-0.5cm}
\caption{Robustness of HesScale-based step-size scaling with AdaHesScale and Adam on $5$ additional MuJoCo environments using A2C. We show the undiscounted return averaged over $10$ independent runs. The shaded area represents the standard error.}
\label{fig:rl-robustness-4}
\end{figure}

\end{document}